%% file: main.tex
\documentclass[letterpaper]{article} 
\usepackage{aaai2027}  
\usepackage[hyphens]{url}  
\usepackage{graphicx} 
\usepackage{natbib}  
\usepackage{caption} 
\usepackage{algorithm}
\usepackage{algorithmic}

\usepackage{newfloat}
\usepackage{listings}
\DeclareCaptionStyle{ruled}{labelfont=normalfont,labelsep=colon,strut=off} 
\floatstyle{ruled}
\newfloat{listing}{tb}{lst}{}
\floatname{listing}{Listing}

\usepackage{graphicx}
\usepackage{amsmath,amssymb,amsfonts}
\usepackage{booktabs}
\usepackage{enumitem}
\usepackage{microtype}
\usepackage{xcolor}
\usepackage{amsthm}
\usepackage{mathtools}

\newtheorem{definition}{Definition}
\newtheorem{theorem}{Theorem}
\newtheorem{proposition}{Proposition}

\newtheorem{hypothesis}{Hypothesis}

\definecolor{figgreen}{RGB}{0,170,85}

\usepackage[most]{tcolorbox}
\usepackage{tabularx}

\usepackage[most]{tcolorbox}

\usepackage{booktabs}
\usepackage{xcolor}

\usepackage{makecell}

\definecolor{routingblue}{HTML}{2F80B7}
\definecolor{filterorange}{HTML}{C06A3B}
\definecolor{valuepurple}{HTML}{744371}

\definecolor{posgreen}{RGB}{0,90,45}
\definecolor{negred}{RGB}{150,25,25}
\definecolor{distgray}{gray}{0.48}

\usepackage{booktabs}

\title{Lifted State Hypothesis in Large Language Models}
\author{
    Bumjin Park,
    Jaesik Choi
}

\affiliations{
    KAIST AI\\
    Daejeon, Republic of Korea\\
    bumjinpark@kaist.ac.kr, jaesik.choi@kaist.ac.kr
}

\begin{document}

\maketitle

\begin{abstract}
Large language models (LLMs) adapt rapidly through fine-tuning and
in-context learning, yet it remains unclear which inputs they treat as the
same case and why their predictions change together. We introduce
the \emph{Lifted State Hypothesis}. Under fixed model parameters, context
scope, and target computation, samples indistinguishable in their observed
computation-relevant behavior form a \emph{computation-relative type}. We
hypothesize that compatible episodes activate a reusable latent
component---a \emph{lifted state}---that supports the target computation.
State reuse enables type-level generalization but creates a non-monotonic
revision problem. When later evidence distinguishes a subtype, revising a
state still shared with its parent may affect members whose predictions
should remain unchanged. The model must separate the subtype through
rerouting, a new state, or input-specific compensation. We formalize this
relation between generalization and revision. We introduce the
\textsc{NMR-Type Dataset} to evaluate LLMs. The dataset first supports a
broad modulo rule. It then provides conflicting supervision for a withheld
subtype while replaying earlier examples. Across full fine-tuning, LoRA, and
in-context learning, models often generalize the broad rule to the subtype
but fail to localize its later revision. These results provide behavioral
evidence consistent with the hypothesis and motivate further study of
lifted-state formation and revision.
\end{abstract}

\input{scripts/introduction}
\input{scripts/methods}
\input{scripts/experiments}
\input{scripts/conclusion}
\bigskip

\bibliography{aaai2027}


\appendix

\input{appendix/1_lifted_state_comparison}
\input{appendix/2_theory}
\input{appendix/3_nmr_type_dataset}
\input{appendix/4_experimental_results}

\end{document}

%% file: scripts/introduction.tex
\begin{center}
\itshape
``no statement is immune to revision.''
\vspace{0.6em}
\normalfont

--- W.V.O Quine
\end{center}
\vspace{0.5em}

\section{Introduction}
\label{sec:introduction}

Large language models (LLMs) acquire broad linguistic, factual, and
behavioral knowledge through pretraining and post-training, and rapidly adapt
through fine-tuning or in-context learning
\cite{NEURIPS2020_1457c0d6,3600270.3602281}. Knowledge acquired from one
example often extends to related cases. This sharing supports efficient
generalization but also exposes multiple predictions to the same update. The
challenge arises when new evidence should revise only part of a shared group.
Understanding adaptation therefore requires explaining \emph{what} changes,
\emph{which inputs} change together, and how their grouping is refined.

Consider Snow White, who first learns not to admit the witch and later learns
to admit a friend wearing the witch's clothes. The second observation should
revise the decision for the friend while preserving the decision for the
witch. This requires refining a previously useful grouping rather than merely
adding another input--output association. The same problem has practical
consequences for LLM safety: further fine-tuning can weaken aligned refusal
behavior, both through low-rank updates and even through benign downstream
data
\cite{lermen2024lora,qi2024finetuning}. Studies of model editing and ripple
effects likewise show that an update can alter related predictions and connect
such propagation to cross-input gradient similarity
\cite{cohen2024ripple,qin2024messy}. These findings characterize the scope and
mechanisms of non-local change, but leave open a prior structural question:
what makes a model treat certain inputs as computationally equivalent, and how
can that equivalence be refined when only a subtype should change?

We address this question through the \emph{Lifted State Hypothesis}. Samples
that are equivalent relative to a context and target computation form a
\emph{computation-relative type}; we call a reusable neural realization of
such a type a \emph{lifted state}. This formulation is motivated by automata,
where inputs with indistinguishable future behavior map to the same state, and
lifted probabilistic inference, where interchangeable instances reuse one
computation until new evidence requires their group to be refined
\cite{hopcroft2001introduction,poole2003first,braz2005lifted}. 
In an LLM, we model a lifted state as a filter--value component. An
episode-conditioned representation interacts with a shared filter, and
episodes that activate the state reuse its value in constructing
representations that support the target computation. Because the underlying representation may be distributed, overlapping, and non-unique, we focus on a minimal formalization of a lifted state that captures its activation and reuse.

This formulation connects type-level generalization to non-monotonic
revision. Learning from diverse members of a parent type can produce a broad
filter that also activates for an aligned but unseen subtype, allowing it to
inherit the shared computation. If the parent and subtype continue to activate
the same state after new evidence distinguishes the subtype, they are
constrained to receive similar changes from that shared component. Selective
revision therefore requires the model to refine their activation structure,
for example by changing the episode representation or filter, introducing an
additional state, or compensating through other components.

To examine the behavioral signature of this predicted coupling, we introduce
the \textsc{NMR-Type Dataset}. The dataset first supports a broad modulo rule
and then introduces conflicting supervision for a withheld subtype while
replaying the previously learned groups. Across full fine-tuning, LoRA, and
in-context learning, models often generalize the initial rule to the unseen
subtype but fail to localize its later revision, altering predictions that
should remain unchanged. These results provide behavioral evidence consistent
with the Lifted State Hypothesis; identifying the corresponding latent states
requires representational and causal analysis.

Our contributions are threefold. First, we formalize computation-relative
types and lifted states as a causal account of reusable computation. Second,
we characterize how a broadly shared state can support generalization while
coupling subsequent changes across a parent type and its subtype. Third, we
introduce the \textsc{NMR-Type Dataset} and demonstrate the predicted
generalization--revision pattern across parameter-based and in-context
learning. Together, these contributions provide a framework for studying how LLMs group
samples into shared computational types and refine those groupings through
generalization and revision.

\section{Related Work}
\label{sec:related-work}

\paragraph{Knowledge revision and update scope.}
Symbolic AI has long studied how new information can revise selected
conclusions while preserving what remains valid, through non-monotonic
reasoning, truth maintenance, and belief revision
\cite{McCarthy1969-MCCSPP,mccarthy1980circumscription,reiter1980logic,
doyle1979truth,alchourron1985logic,darwiche1997logic}. In neural models,
continual learning uses regularization or replay to reduce interference,
whereas model editing modifies selected parametric associations
\cite{kirkpatrick2017overcoming,lopez2017gradient,
de2021editing,meng2022locating,mitchell2022fast}. Recent studies further
distinguish intended ripple effects from collateral changes, relate their
propagation to gradient similarity, and examine rule-level generalization
\cite{cohen2024ripple,qin2024messy,zhou2025ruleedit}. These works characterize the scope and propagation of revision, whereas we
study the shared computational organization from which that scope may arise.

\paragraph{Computational equivalence and lifting.}
Automata map inputs with identical future behavior to the same state, while
predictive representations and state abstractions preserve only distinctions
needed for future prediction or decision making
\cite{hopcroft2001introduction,littman2001predictive,
givan2003equivalence,pmlr-v48-abel16}. Lifted probabilistic inference similarly
groups interchangeable ground instances and reuses one computation across
them
\cite{poole2003first,braz2005lifted,milch2008lifted}. When evidence breaks
this interchangeability, lifted methods refine the grouping through splitting,
shattering, or partial grounding
\cite{pmlr-v22-taghipour12}. These frameworks provide explicit precedents for
equivalence-based reuse and subsequent refinement.

\paragraph{Reusable computation in LLM representations.}
Interpretability studies identify several candidate forms of reusable
computation. Transformer MLPs have been modeled as key--value memories, while
knowledge neurons and distributed features have been linked to factual
associations
\cite{geva-etal-2021-transformer,dai-etal-2022-knowledge}. Task vectors,
function vectors, and residual-stream belief-state geometry further show that
compact or distributed representations can support computations across
multiple contexts
\cite{hendel2023taskvectors,todd2024functionvectors,shai2024beliefstate}.
The Lifted State Hypothesis provides a functional account of how such reusable representations are formed and shared across inputs.

%% file: scripts/methods.tex
\section{Computation-Relative Types and \\Lifted States}
\label{sec:types-and-lifted-states}

In this section, we first define \emph{computation-relative types} as
extensional equivalence classes induced by behavior relevant to a target
computation. We then discuss conceptual precedents in automata theory and
lifted probabilistic inference, which motivate equivalence-based reuse and
refinement. Finally, we present the \emph{Lifted State Hypothesis} as a neural
account of how episodes within such types may reuse computation under minimal
representational assumptions: an input-conditioned representation, a filter,
and an output representation.

\subsection{Computational Types}
\label{sec:computational-types}

Let \(\mathcal{C}\), \(\mathcal{X}\), and \(\mathcal{F}\) denote the spaces
of contexts, samples, and target computations, respectively. A computation
\(f\in\mathcal{F}\) specifies an analyst-selected functional role under which
model behavior is evaluated; it need not be explicitly supplied to the model.
Under parameters \(\theta\), it induces a context-dependent mapping
\begin{equation}
    f_{\theta}:
    \mathcal{C}\times\mathcal{X}
    \rightarrow
    \mathcal{Z}_{f},
    \label{eq:computation-map}
\end{equation}
where \(\mathcal{Z}_{f}\) is the space of computation outcomes, such as model
outputs, probability distributions, or internal representations.

Let
\(
    \rho_f:\mathcal{Z}_f\rightarrow\mathcal{V}_f
\)
be an observation map into a space \(\mathcal{V}_f\) that retains only the
distinctions relevant to \(f\). For example, \(\rho_f\) may extract a discrete
decision, a task-relevant statistic, or the functional consequence of an
internal representation. Given a context scope
\(\Gamma\subseteq\mathcal{C}\), two samples are
\emph{computationally equivalent} when their observed \(f\)-relevant behavior
agrees throughout \(\Gamma\):
\begin{equation}
\begin{aligned}
    x\equiv_{\theta,\Gamma,f}x'
    \quad\Longleftrightarrow\quad
    \rho_f\!\left(f_{\theta}(x;c)\right)
    &=
    \rho_f\!\left(f_{\theta}(x';c)\right),\\[-0.2em]
    &\forall c\in\Gamma.
\end{aligned}
\label{eq:computational-equivalence}
\end{equation}
The corresponding \emph{computation-relative type} is the equivalence class
\begin{equation}
    \tau_{\theta,\Gamma,f}(x)
    =
    [x]_{\equiv_{\theta,\Gamma,f}}.
    \label{eq:computational-type}
\end{equation}

A type is therefore relative to the current parameters \(\theta\), context
scope \(\Gamma\), and target computation \(f\), rather than an intrinsic
semantic category. The same sample may belong to different types under
different contexts or computations, and learning may alter the partition by
changing \(\theta\). Conversely, surface-dissimilar samples may share a type
when they are indistinguishable for the computation under study. A shared
output label alone is insufficient when the samples differ in other
computation-relevant behavior.

\subsection{Computational Precedents}
\label{sec:computational-precedents}

Two exact computational settings illustrate how equivalence classes support
reusable state-level computation. They serve as precedents for the hypothesis
below, not as architectural claims about LLMs.

\paragraph{Automata.}
For a language \(L\subseteq\Sigma^{*}\), let \(x\) be an input prefix and
\(w\in\Sigma^{*}\) a continuation. The relevant computation is the future
acceptance decision
\begin{equation}
    f_L(x;w)
    =
    \mathbb{I}[xw\in L].
    \label{eq:automata-computation}
\end{equation}
Two prefixes are equivalent when no continuation distinguishes them:
\begin{equation}
    x\equiv_L x'
    \quad\Longleftrightarrow\quad
    f_L(x;w)=f_L(x';w),
    \quad
    \forall w\in\Sigma^{*}.
    \label{eq:myhill-nerode-equivalence}
\end{equation}
By the Myhill--Nerode theorem, when \(L\) is regular, these equivalence
classes correspond exactly to the states of the minimal deterministic
finite automaton recognizing \(L\)
\cite{hopcroft2001introduction}. Each state therefore preserves precisely
the distinctions required for future acceptance; once two prefixes reach the
same state, their remaining identities are irrelevant to that computation.

\paragraph{Lifted probabilistic inference.}
Lifted probabilistic inference reduces repeated computation in relational
probabilistic models by grouping ground instances that remain interchangeable
under the current evidence, constraints, and query
\cite{poole2003first,braz2005lifted}. Rather than evaluating each grounding
separately, it performs a shared computation for the corresponding lifted
group. Let \(\tau\) denote such a group, \(n_\tau=|\tau|\) its grounding
count, and \(\psi_\tau(y;c)\) the common factor contributed by each member
to the target computation under context \(c\).

\begin{proposition}[Lifted sufficiency and refinement]
\label{prop:lifted-sufficiency-refinement}
Fix a lifted partition, and suppose every \(x_i\in\tau\) contributes the same
factor \(\psi_\tau(y;c)\). Then the count \(n_\tau\) and shared factor
\(\psi_\tau\) are sufficient to compute the repeated factor contribution:
\begin{equation}
    \prod_{x_i\in\tau}
    \psi_\tau(y;c)
    =
    \psi_\tau(y;c)^{n_\tau}.
    \label{eq:lifted-group-contribution}
\end{equation}
If later evidence partitions \(\tau\) into subgroups requiring different
factors, the original lifted summary is insufficient to determine their
assignments and separate contributions.
\end{proposition}

Exact lifted methods therefore use the available evidence and constraints to
split, shatter, or partially ground a group when its previous
interchangeability no longer holds
\cite{milch2008lifted,pmlr-v22-taghipour12}. The common principle is that a
shared representation need preserve only the distinctions required by the
current computation, whereas later refinement depends on information outside
that shared summary.

Motivated by this principle, we hypothesize an analogous organization in
LLMs: a lifted state carries computation shared across compatible episodes,
while episode-conditioned representations or other latent components preserve
information that can distinguish them. Selective revision may use this
information to refine the state's filtering or introduce a separate state.
A detailed account of lifted states in automata theory and lifted probabilistic inference is provided in Appendix~\ref{appendix:1_lifted_states}.

\subsection{Lifted State Hypothesis in Language Models}
\label{sec:lifted-state-hypothesis}

Automata and lifted probabilistic inference provide explicit realizations of equivalence-based reuse. LLM representations, however, are continuous, distributed, and generally non-identifiable. We therefore ask what minimal representational structure is sufficient to satisfy the functional conditions of a lifted state across multiple episodes.

Consider an episode
\(
    e_i=(c_i,x_i,f_i),
\)
where \(c_i\in\mathcal{C}\), \(x_i\in\mathcal{X}\), and
\(f_i\in\mathcal{F}\). The computation \(f_i\) specifies the functional role
under which the episode is analyzed and need not appear explicitly in the
input. Fix an internal representation space
\(
    \mathcal{M}\subseteq\mathbb{R}^{d}
\)
and let
\(
    h_i=h_{\mathcal{M}}(e_i;\theta)\in\mathcal{M}.
\)
We use the following idealized local decomposition:
\begin{equation}
\begin{aligned}
    s_k
    &=
    (z_k,\phi_k),
    \qquad
    a_k(e_i)
    =
    \sigma\!\left(z_k^{\top}q(e_i)\right),\\
    h_i
    &=
    \sum_{k=1}^{K}
    a_k(e_i)\phi_k
    +
    r_i.
\end{aligned}
\label{eq:lifted-state-composition}
\end{equation}
Here, \(q(e_i)\) is an episode-conditioned representation,
\(z_k\) is a shared filter, and \(\phi_k\in\mathcal{M}\) is the value
contributed by the state. The residual \(r_i\) collects episode-specific and
other unmodeled contributions. Thus, episodes may reuse the same lifted state
while retaining distinctions that can later support selective refinement.
This decomposition is a functional approximation rather than a unique
factorization of the representation.

Associate each candidate state \(s_k\) with a computation-relative type
\(\tau_k\subseteq\mathcal{X}\), context scope
\(\Gamma_k\subseteq\mathcal{C}\), and target computation
\(f_k\in\mathcal{F}\). Define its functional scope as
\begin{equation}
    \operatorname{Scope}(s_k)
    :=
    \tau_k \times \Gamma_k \times \{f_k\}.
    \label{eq:state-scope}
\end{equation}
An episode \(e_i=(x_i,c_i,f_i)\) lies within this scope when
\begin{equation}
\begin{aligned}
    e_i \in \operatorname{Scope}(s_k) ~
    \Longleftrightarrow~&
    (x_i \in \tau_k) \land (c_i \in \Gamma_k)
    \\
    &\hspace{2.5em}{}\land (f_i = f_k).
\end{aligned}
\label{eq:state-scope-membership}
\end{equation}
\begin{definition}[Functional reliance on a latent state]
\label{def:state-reliance}
An episode \(e_i\) functionally relies on a candidate state \(s_k\), written
\(e_i \models s_k\), when the episode lies within the functional scope of
\(s_k\) and uses the state in realizing the target computation \(f_k\):
\begin{equation}
    e_i \models s_k
    ~\Longleftrightarrow~
    \bigl(e_i \in \operatorname{Scope}(s_k)\bigr)
    \land
    \bigl(a_k(e_i)>0\bigr).
    \label{eq:state-reliance}
\end{equation}
Here, \(a_k(e_i)\) denotes the contribution weight of \(s_k\) to the
representation used for \(f_k\), and
\(v_k(e_i)=a_k(e_i)\phi_k\) denotes the corresponding state-specific
component.
\end{definition}

\begin{hypothesis}[Lifted State Hypothesis]
\label{hyp:lifted-state}
An LLM realizes some recurring computations through reusable latent states.
A state \(s_k=(z_k,\phi_k)\) is \emph{lifted} when episodes within a
computation-relative type, context scope, and target computation
systematically satisfy \(e_i\models s_k\) and reuse the same value
\(\phi_k\). For each such episode, the state contributes
\(v_k(e_i)=a_k(e_i)\phi_k\) to the representation supporting the target
computation. A lifted state may be distributed, overlapping, graded, and
non-unique.
\end{hypothesis}

The relation \(e_i\models s_k\) holds when the episode lies within the
functional scope of \(s_k\) and activates the state in realizing the target
computation. Because multiple episodes may rely on the same state, modifying
\(s_k\) may induce coordinated changes across its functional scope.

For a fixed target computation \(f\), two episodes belong to the same
computational type when they have the same computation-relevant outcome:
\begin{equation}
    e_i \equiv_{\theta,f} e_j
    ~\Longleftrightarrow~
    f_{\theta}(e_i)=f_{\theta}(e_j).
    \label{eq:episode-equivalence}
\end{equation}
Because an episode includes its context, the same sample may belong to
different types as the context changes. For example, a bird may satisfy
\textsc{CanFly} when healthy but \textsc{CannotFly} when injured, even under
the same target computation. An episode may also participate in several computations. For example, a bird
may satisfy both \textsc{CanFly} and \textsc{CanRun} under distinct target
computations through different latent states. 

In this work, since computational equivalence is defined by equality of outcomes under a selected computation \(f\), we analyze lifted states relative to a single fixed \(f\). A group of episodes may nevertheless be jointly characterized by multiple computations, which we leave to future work.

\section{Non-Monotonic Revision through \\ Lifted States}
\label{sec:non-monotonic-revision}

Lifted-state sharing becomes a non-monotonic revision problem when the initial
evidence supports a broad parent rule but does not reveal a relevant subtype.
We first describe how such incomplete evidence can cause an unseen subtype to
inherit the parent computation, and then analyze why its later revision may
remain coupled to preserved parent members.

\subsection{Subtype Generalization from Incomplete Evidence}
\label{sec:subtype-generalization}

Let \(\mathcal{P}\subseteq\mathcal{X}\) be a task-level parent set and
\(\mathcal{S}\subsetneq\mathcal{P}\) a subtype that later requires a different
outcome. Let \(y^{+}\) be the parent outcome and
\(y^{-}\neq y^{+}\) the subtype outcome. Writing
\(
    \mathcal{X}_1
    =
    \operatorname{supp}_{\mathcal{X}}(\mathcal{D}_1)
\),
we assume that no subtype example appears in the initial data:
\begin{equation}
    \mathcal{X}_1\cap\mathcal{S}
    =
    \varnothing.
    \label{eq:subtype-incomplete-evidence}
\end{equation}
The observations are therefore consistent with both a coarse parent rule and
a refined subtype rule:
\begin{equation}
\begin{aligned}
    \text{coarse:}\quad
    &\mathcal{P}
    \mapsto
    y^{+},
    \\
    \text{refined:}\quad
    &\mathcal{P}\setminus\mathcal{S}
    \mapsto
    y^{+},
    \qquad
    \mathcal{S}
    \mapsto
    y^{-}.
\end{aligned}
\label{eq:initial-rule-ambiguity}
\end{equation}
Because the rules agree on all observed parent examples,
\(\mathcal{D}_1\) does not identify \(\mathcal{S}\) as a separate type.

Incomplete evidence alone does not determine which compatible rule is
learned. We consider learners with a \emph{default generalization bias}:
absent distinguishing evidence, learning favors the coarser extension of the
observed parent outcome. This resembles default reasoning, in which a general
conclusion is retained until defeating evidence appears
\cite{reiter1980logic}.

Let
\(
    e=(c,x,f)
\)
with
\(
    x\in\mathcal{P}\setminus\mathcal{S}
\),
and
\(
    e'=(c,x',f)
\)
with
\(
    x'\in\mathcal{S}
\).
Under the default extension,
\begin{equation}
    \widehat{y}_1(e)
    =
    \widehat{y}_1(e')
    =
    y^{+}.
    \label{eq:subtype-broad-generalization}
\end{equation}
If their observed \(f\)-relevant behavior agrees throughout \(\Gamma\), then
\(
    x\equiv_{\theta_1,\Gamma,f}x'
\).
The model has thus incorporated the unseen subtype into a coarser
computation-relative type through a defeasible inference from incomplete evidence.

Suppose the observed parent episodes
\(e_1,\ldots,e_n\) use a lifted state
\(s_k=(z_k,\phi_k)\), with activations
\begin{equation}
    a_{k,t}(e_i)
    =
    \sigma\!\left(
        z_{k,t}^{\top}q_t(e_i)
    \right),
    \qquad
    i=1,\ldots,n.
    \label{eq:nmr-parent-state-activation}
\end{equation}
Before parent learning, the unseen subtype episode \(e'\) may activate the
same state positively but less strongly:
\begin{equation}
    0
    <
    a_{k,0}(e')
    <
    \min_{1\leq i\leq n}
    a_{k,0}(e_i).
    \label{eq:weak-subtype-activation}
\end{equation}
Thus, the subtype is initially aligned with the parent state without yet
using it as strongly as the observed parent episodes.

To isolate a filter-mediated mechanism, hold the episode representations fixed
during the initial stage and define the filter change induced by parent
learning as
\begin{equation}
    \Delta_1 z_k
    :=
    z_{k,1}-z_{k,0}.
    \label{eq:initial-filter-change}
\end{equation}
For the subtype episode, let
\begin{equation}
    \ell_{k,t}(e')
    =
    z_{k,t}^{\top}q(e')
    \label{eq:subtype-activation-logit}
\end{equation}
denote its activation logit. The induced change is
\begin{equation}
\begin{aligned}
    \Delta_1\ell_k(e')
    &:=
    \ell_{k,1}(e')-\ell_{k,0}(e')
    \\
    &=
    q(e')^{\top}\Delta_1 z_k.
\end{aligned}
\label{eq:subtype-logit-transfer}
\end{equation}

\begin{proposition}[Subtype activation transfer]
\label{prop:subtype-activation-transfer}
If \(\sigma\) is strictly increasing and
\begin{equation}
    q(e')^{\top}\Delta_1 z_k
    >
    0,
    \label{eq:positive-subtype-transfer}
\end{equation}
then
\begin{equation}
    a_{k,1}(e')
    >
    a_{k,0}(e').
    \label{eq:subtype-activation-growth}
\end{equation}
\end{proposition}
The proposition identifies a sufficient condition for parent learning to
increase the subtype's activation of \(s_k\); activation growth alone does not
establish that the subtype effectively uses the state. We therefore restrict
the subsequent analysis to the case in which, after learning
\(\mathcal{D}_1\),
\begin{equation}
    e\models s_k,
    \qquad
    e'\models s_k.
    \label{eq:shared-state-after-parent-learning}
\end{equation}
This condition permits different activation strengths and
episode-specific representations. It requires only that the same state make a
computation-relevant contribution to both episodes. The following subsection
examines the consequences if this shared dependence persists when
subtype-specific evidence is introduced. A formal analysis of subtype activation transfer is provided
in Appendix~\ref{appendix:theoretical-analysis}.

\subsection{Non-Monotonic Revision under Persistent Sharing}
\label{sec:revision-under-persistent-sharing}

The second-stage data \(\mathcal{D}_2\) introduce the subtype distinction
absent from \(\mathcal{D}_1\). Selective revision must change the inherited
subtype prediction while preserving the parent outcome:
\begin{equation}
\begin{aligned}
    \widehat{y}_2(e')
    &=
    y^{-}
    \neq
    \widehat{y}_1(e'),
    \\
    \widehat{y}_2(e)
    &=
    y^{+}
    =
    \widehat{y}_1(e).
\end{aligned}
\label{eq:desired-subtype-revision}
\end{equation}
Accordingly, the computational equivalence induced under incomplete evidence
must be refined:
\begin{equation}
    x\equiv_{\theta_1,\Gamma,f}x',
    \qquad
    x\not\equiv_{\theta_2,\Gamma,f}x'.
\label{eq:required-type-refinement}
\end{equation}
The revision is non-monotonic because subtype evidence defeats a prediction
previously supported by the default extension of the parent rule.

The central question is whether this distinction can be learned while the
parent and subtype continue to depend on the same lifted state. For either
episode \(\bar e\in\{e,e'\}\), define the contribution of \(s_k\) at stage
\(t\) as
\begin{equation}
\begin{aligned}
    v_{k,t}(\bar e)
    &:=
    a_{k,t}(\bar e)\phi_{k,t},
    \\
    \Delta v_k(\bar e)
    &:=
    v_{k,2}(\bar e)-v_{k,1}(\bar e).
\end{aligned}
\label{eq:revision-state-contribution}
\end{equation}
Suppose their activation discrepancy remains bounded before and after subtype
learning:
\begin{equation}
    \left|a_{k,t}(e)-a_{k,t}(e')\right|
    \leq
    \epsilon_{k,t},
    \quad
    t\in\{1,2\}.
\label{eq:persistent-activation-condition}
\end{equation}

\begin{proposition}[Persistent sharing bounds differential revision]
\label{prop:persistent-sharing-coupling}
Under Equation~\eqref{eq:persistent-activation-condition},
\begin{equation}
    \left\|\Delta v_k(e)-\Delta v_k(e')\right\|
    \leq
    \epsilon_{k,2}\|\phi_{k,2}\|
    +
    \epsilon_{k,1}\|\phi_{k,1}\|.
\label{eq:persistent-sharing-bound}
\end{equation}
\end{proposition}

The proposition states that episodes maintaining similar activation of the
same state cannot receive substantially different changes through that state.
It concerns only the shared contribution \(v_k\), rather than the complete
representations or final predictions. Nevertheless, when the downstream
effect of \(s_k\) is similar for the two episodes, this component-level
coupling creates pressure for their outputs to move in the same direction.

Selective revision must either separate the subtype from the shared state or
control how the state changes across its episodes. For
\(\bar e\in\{e,e'\}\), define
\begin{equation}
\begin{aligned}
    \textcolor{routingblue}{\Delta q(\bar e)}
    &:=
    q_2(\bar e)-q_1(\bar e), \\
    \textcolor{filterorange}{\Delta z_k}
    &:=
    z_{k,2}-z_{k,1},
    \\
    \textcolor{valuepurple}{\Delta\phi_k}
    &:=
    \phi_{k,2}-\phi_{k,1}.
\end{aligned}
\label{eq:revision-component-changes}
\end{equation}
Writing
\[
    \beta_{k,1}(\bar e)
    :=
    \sigma'\!\left(
        z_{k,1}^{\top}q_1(\bar e)
    \right),
\]
the state change has the first-order decomposition
\begin{equation}
\begin{aligned}
    \Delta v_k(\bar e)
    \approx{}&
    \underbrace{
        \beta_{k,1}(\bar e)
        z_{k,1}^{\top}
        \textcolor{routingblue}{\Delta q(\bar e)}
        \phi_{k,1}
    }_{\text{episode-specific}}
    \\
    &+
    \underbrace{
        \beta_{k,1}(\bar e)
        q_1(\bar e)^{\top}
        \textcolor{filterorange}{\Delta z_k}
        \phi_{k,1}
    }_{\text{shared filter}}
    +
    \underbrace{
        a_{k,1}(\bar e)
        \textcolor{valuepurple}{\Delta\phi_k}
    }_{\text{shared value}}.
\end{aligned}
\label{eq:three-revision-channels}
\end{equation}

The episode-specific change
\(\textcolor{routingblue}{\Delta q(\bar e)}\) alters the representation that
each episode presents to the shared filter, thereby increasing or decreasing
its use of \(s_k\) independently of other episodes. The shared change
\(\textcolor{filterorange}{\Delta z_k}\) instead modifies the common
activation criterion across the lifted episodes. It does not directly change
the state's value, but adjusts which episodes receive that value and how
strongly, using distinctions already expressed in their representations. In
contrast, \(\textcolor{valuepurple}{\Delta\phi_k}\) changes the
output-relevant value itself. Episodes that remain active on \(s_k\) therefore
receive the same revised direction, scaled by their activation strengths.

Selective revision may thus separate the subtype through episode-specific
rerouting by \(\textcolor{routingblue}{\Delta q(\bar e)}\) or through shared
filter refinement by \(\textcolor{filterorange}{\Delta z_k}\). If these
changes do not sufficiently reduce the subtype's sharing with preserved parent
episodes, revising \(\textcolor{valuepurple}{\Delta\phi_k}\) propagates the
change across those episodes as well. The next section examines the resulting
tension between subtype revision and parent retention. Further details are
provided in Appendix~\ref{appendix:theoretical-analysis}.

%% file: scripts/experiments.tex
\section{Experiments}
\label{sec:experiments}

\subsection{NMR-Type Dataset}
\label{sec:nmr-type-benchmark}
\paragraph{Task design.}
We introduce the \textsc{NMR-Type Dataset}, a controlled two-stage task for
testing whether a model can revise a subtype while preserving previously
learned predictions. Let
\(
    P_k(x)\coloneqq\mathbb{I}[k\mid x]
\)
denote divisibility by \(k\). In \(W_1\), the model observes examples
supporting three task rules:
\begin{equation}
\begin{aligned}
    P_2(x)
    &\mapsto
    \texttt{positive},
    \qquad
    P_3(x)
    \mapsto
    \texttt{negative},
    \\
    P_5(x)
    &\mapsto
    \texttt{positive}.
\end{aligned}
\label{eq:w1-rules}
\end{equation}
Multiples of \(4\) are withheld from this stage. Since
\(P_4(x)\Rightarrow P_2(x)\), predicting \(\texttt{positive}\) for the
withheld group reflects generalization of the observed \(P_2\) rule.

In \(W_2\), the withheld subtype is assigned the revised rule
\begin{equation}
    P_4(x)
    \mapsto
    \texttt{negative},
\label{eq:w2-subtype-rule}
\end{equation}
while examples from the \(P_2\), \(P_3\), and \(P_5\) groups are replayed
with their original labels. The model must therefore revise
\(\mathcal{G}_4\) from positive to negative while preserving the labels of
\(\mathcal{G}_2\), \(\mathcal{G}_3\), and \(\mathcal{G}_5\). 
To keep the
modulo-defined groups disjoint, we exclude inputs that are multiples of more
than one group modulus. For example, when groups are defined by divisibility
by \(2\) and \(5\), their common multiples, such as \(10\), are excluded.
Exact disjoint
group definitions are provided in Appendix~\ref{appendix:3_datasets}.

\paragraph{Evaluation.}
We report accuracy separately for each task-defined group. For group
\(\mathcal{G}_j\) at stage \(t\), we define
\begin{equation}
    \operatorname{Acc}_{t}(\mathcal{G}_j)
    =
    \Pr\!\left[
        \widehat{y}_{\theta_t}(x)=y_{j,t}
        \mid x\in\mathcal{G}_j
    \right],
    \label{eq:nmr-group-accuracy}
\end{equation}
where \(y_{j,t}\) is its target label at that stage. 


\subsection{Learning Protocols}
\label{sec:experimental-protocols}

\paragraph{Sequential setup.}
We construct disjoint training and test pools from integers in
\(\{1,\ldots,1000\}\). Each stage contains \(32\) training examples with
balanced labels. Training proceeds sequentially as
\(W_1\rightarrow W_2\), without resetting the model parameters. During
\(W_2\), newly introduced \(\mathcal{G}_4\) examples are combined with replay
examples from the three earlier groups. We resample training examples within
each group across five seeds and report mean group-wise test accuracy.

\paragraph{Models and update methods.}
We evaluate
Gemma-3-\(4\mathrm{B}/12\mathrm{B}/27\mathrm{B}\)
\cite{gemmateam2025gemma3technicalreport},
Qwen3-\(4\mathrm{B}/14\mathrm{B}/32\mathrm{B}\)
\cite{yang2025qwen3technicalreport}, and
Llama-3.1-\(8\mathrm{B}\)
\cite{grattafiori2024llama3herdmodels}.
We consider full fine-tuning, LoRA fine-tuning
\cite{hu2022lora}, and in-context learning to examine whether revision and retention differ across model families, and scales.
Detailed experimental settings, including optimizer and adapter
settings, are reported in Appendix~\ref{appendix:experimental_settings}.

\begin{figure*}[t]
    \centering
    \includegraphics[width=1.0\linewidth]
    {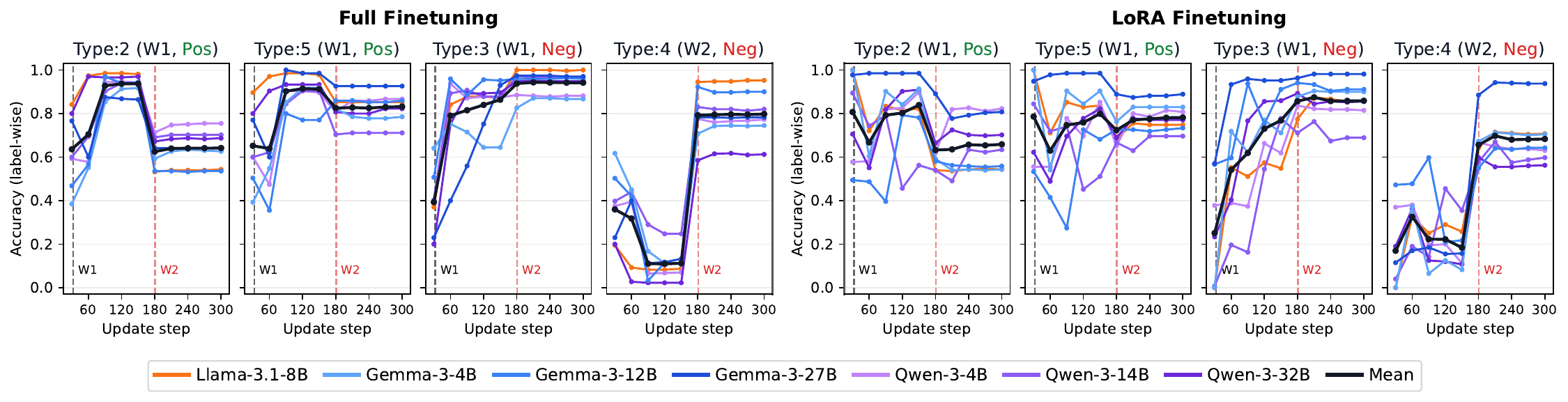}
\caption{Label-wise test accuracy during sequential \(W_1 \rightarrow W_2\)
learning with full fine-tuning and LoRA. The dashed line marks the introduction
of negative \(\mathcal{G}_4\) examples. Low \(\mathcal{G}_4\) accuracy during
\(W_1\) indicates generalization of the positive parent rule. After the
transition, \(\mathcal{G}_4\) improves sharply, while
\(\mathcal{G}_2\) and \(\mathcal{G}_5\) decline and partially recover.
The larger change on \(\mathcal{G}_2\) is consistent with revision of the
shared \(P_2\) rule, whereas the smaller change on the disjoint
\(\mathcal{G}_5\) control reflects broader interference.}
    \label{fig:main_fig_tuning}
\end{figure*}

\begin{figure*}[t]
    \centering
    \includegraphics[width=1.0\linewidth]
    {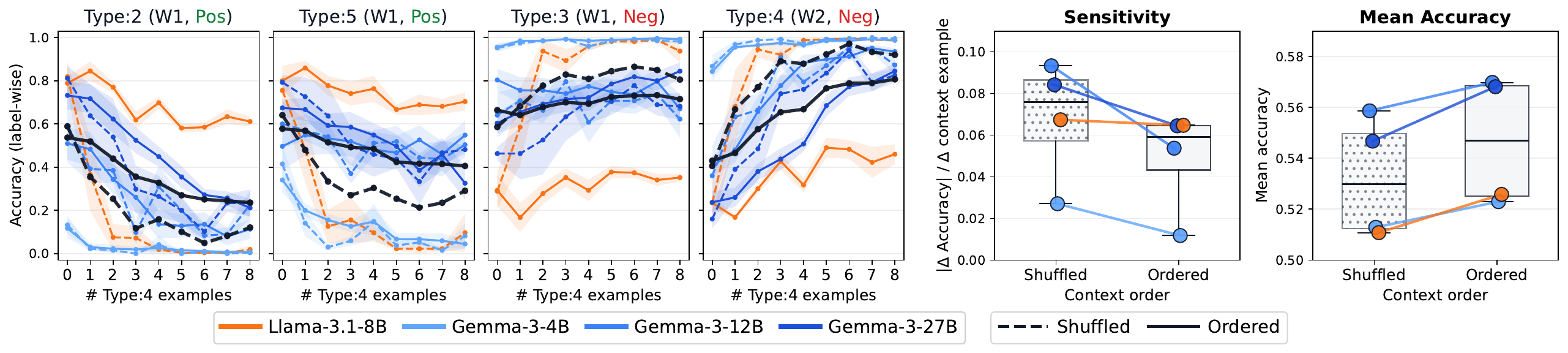}
\caption{In-context revision as \(0\)--\(8\) negative
\(\mathcal{G}_4\) demonstrations are added to the \(W_1\) context.
Dashed and solid curves denote shuffled and learning-stage-ordered contexts,
respectively. Increasing subtype evidence improves \(\mathcal{G}_4\) while
reducing accuracy on the preserved positive groups
\(\mathcal{G}_2\) and \(\mathcal{G}_5\). The last two panels report the sensitivity of accuracy per added demonstration and mean accuracy, respectively.
Shuffled contexts exhibit larger sensitivity, whereas ordered contexts achieve
higher mean accuracy.}
    \label{fig:main_fig_icl}
\end{figure*}
\section{Results}
\label{sec:results}

We examine whether the initial parent rule generalizes to the withheld subtype
and how the subsequent revision propagates across preserved groups. In
particular, we compare the preserved parent group \(\mathcal{G}_2\) with the
disjoint positive control \(\mathcal{G}_5\) to distinguish type-related
revision from broader same-label interference. This comparison tests whether
prediction changes follow the parent--subtype structure rather than the shared
output label alone.

\subsection{Parameter-Based Non-Monotonic Revision}
\label{sec:parameter-based-revision}

Figure~\ref{fig:main_fig_tuning} shows the dynamics under full fine-tuning and
LoRA. During \(W_1\), models generally learn the observed groups and extend
the positive \(P_2\) rule to the withheld subtype \(\mathcal{G}_4\). Because
\(\mathcal{G}_4\) is evaluated against its final negative label, this
generalization appears as low accuracy before \(W_2\).

After negative \(\mathcal{G}_4\) examples are introduced, most models improve
on the revised subtype while retaining \(\mathcal{G}_3\). Accuracy also
decreases on both preserved positive groups despite balanced replay, but not
to the same extent. The decline is larger on \(\mathcal{G}_2\), which
constitutes the preserved remainder of the \(P_2\) parent type, than on
\(\mathcal{G}_5\), which shares the positive label but is disjoint from the
parent--subtype relation. This asymmetry appears under both full fine-tuning
and LoRA.

The results therefore distinguish type-related revision from broader
same-label interference. Prediction changes are concentrated within the
computation-relative type \(\mathcal{G}_2\) implicated by the exception \(\mathcal{G}_4\), whereas changes to the
disjoint positive control (type \(\mathcal{G}_5\)) are smaller and model-dependent. Indeed,
Gemma-3-\(12\mathrm{B}\) improves on the control under full fine-tuning,
showing that its degradation is not required for learning the exception.
This structured revision pattern is behaviorally consistent with
Proposition~\ref{prop:persistent-sharing-coupling} and supports the refinement
of computation-relative types in LLMs.

\subsection{In-Context Non-Monotonic Revision}
\label{sec:icl-revision}
Figure~\ref{fig:main_fig_icl} shows that non-monotonic revision can also arise
through in-context learning. As performance improves on the revised
subtype, accuracy declines on the preserved parent group
\(\mathcal{G}_2\). The smaller decline on the disjoint
\(\mathcal{G}_5\) control is consistent with a broader bias toward the revised
label, a known source of instability in in-context learning
\cite{zhao2021calibrate}. However, the larger decrease
in \(\mathcal{G}_2\) suggests that label bias alone does not explain the
effect; revision is more strongly coupled through the parent--subtype
relation. 

The comparison between context organizations suggests that how demonstrations
are learned in context shapes the type structure induced by the model.
Learning-stage ordering yields higher overall performance and smaller accuracy
gradients than shuffling, indicating more stable subtype revision.

Across full fine-tuning, LoRA, and in-context learning, the scope of revision
consistently follows the model's computation-relative type structure. This
pattern provides behavioral evidence that LLMs form and refine such types as
they incorporate new evidence. Its recurrence across distinct update
mechanisms further suggests that coupled revision reflects a general property
of computation reuse rather than a method-specific artifact. Further discussion of these results is provided in
Appendix~\ref{appendix:experimental_results}.

%% file: scripts/conclusion.tex
\section{Discussion}
\label{sec:discussion}

\paragraph{Compression and shared computation.}
Our results highlight a tension between compression and revision. Neural
language models compress regularities from massive corpora into a finite
parameter space, allowing factual and relational information to be reused
without an explicit knowledge base
\cite{petroni-etal-2019-language,roberts-etal-2020-much}.
Lifted probabilistic inference makes this computational economy explicit:
interchangeable instances share one computation, but the group must be refined
once new evidence distinguishes its members
\cite{poole2003first,braz2005lifted,pmlr-v22-taghipour12}.
The same reuse that supports efficient generalization can therefore make later
revision difficult.

In LLMs, groupings are neither explicit nor localized. Interpretability
studies find knowledge and computation distributed across MLP memories,
knowledge neurons, and sparse representations
\cite{geva-etal-2021-transformer,dai-etal-2022-knowledge,
elhage2022toy,bricken2023monosemanticity,templeton2024scaling}.
Skill neurons, task and function vectors, and belief-state geometry suggest
that reusable computations arise through parameter learning or context
\cite{wang-etal-2022-finding-skill,hendel2023taskvectors,
todd2024functionvectors,shai2024beliefstate}.
The lifted-state hypothesis adds a functional criterion: episodes share a
state when treated as equivalent for a target computation.

\paragraph{Non-monotonic revision and representational plasticity.}
New evidence may show that only a subtype should change, but it does not tell
the model whether to modify a shared computation or create a separate one. 
Default reasoning and belief revision impose the same selective
requirement: contradicted conclusions should be withdrawn while unaffected
ones remain valid
\cite{reiter1980logic,doyle1979truth,alchourron1985logic}.
From the lifted-state perspective, revision therefore requires separating the
updated subtype from cases that should remain unchanged, through rerouting,
filter refinement, or a new state.

Separating a revised subtype from its parent type depends on representational
plasticity. Initial training examples may induce an overly broad grouping that
later subtype-specific supervision fails to refine
\cite{jones2022capturing,wan2025unveiling,
pmlr-v162-nikishin22a}. Continued training may further reduce the model's
capacity to reorganize it
\cite{dohare2024loss,sokar2023dormant}. Failed revision may therefore reflect
both coarse computational types and limited plasticity.

\paragraph{Reasoning-level revision.}
The same pattern extends to reasoning: a model may induce a rule, apply it to
unseen cases, and later narrow its scope when an exception appears. Prior work
examines these processes through rule induction, defeasible inference, and
belief revision
\cite{li2025patterns,rudinger2020thinking,wilie2024belief}. 
Viewed through lifted states, these results suggest examining whether premise
types that share a rule also reuse the same conclusion-producing computation,
and how that sharing changes when the rule is revised. Belief memory, editing,
and unlearning provide complementary ways to maintain consistency or update
selected behaviors
\cite{kassner2021beliefbank,yao2024large,liu2025rethinking}.
The lifted perspective offers a common lens for comparing these mechanisms 
how revised evidence propagates through rules, conclusions, and their applicability.

\paragraph{From model states to AI agent reasoning.}
For an AI agent, revising a belief may require more than changing a single model
output. The agent may also need to update the goals, plans, and actions that
were based on that belief, as emphasized by the Belief--Desire--Intention
framework \cite{rao1995bdi}. Modern LLM agents deepen this dependency by
repeatedly integrating observations, reasoning, memory, and action over long
horizons
\cite{yao2023react,yanfangzhou-etal-2025-m2pa,yang-etal-2025-coarse}.
The lifted-state perspective therefore asks whether new evidence creates the
computational distinctions needed for revision or merely reinforces an
existing shared interpretation. 
This issue is particularly important in multi-agent settings. An agent must
evaluate information that may rely on different premises and support
conclusions at different levels of abstraction. It must therefore determine
whether two claims concern the same type, a more specific subtype, or
fundamentally different cases before deciding which beliefs should be revised
\cite{jones2022capturing,wan2025unveiling,sharma2024towards}.

For AI alignment, revision must therefore be both responsive and selective:
reliable evidence should update the affected beliefs, plans, and actions
without destabilizing commitments supported by unaffected evidence
\cite{zhao2023large}. This requires the agent to represent
not only what should change, but also how far that change should propagate.


\paragraph{Limitations and future work.}
Our experiments provide behavioral evidence but do not directly characterize
how lifted states are represented or realized within the model. Moreover, the
controlled dataset abstracts away from biases and associations inherited
from pretraining and real-world data, which may shape how computational types
are initially formed and later revised. Future work should identify the
representational and causal basis of lifted states, examine their revision
under naturalistic data distributions, and extend the framework to
reasoning- and agent-level settings.

\section{Conclusion}
\label{sec:conclusion}
Understanding change in LLMs requires identifying the internal states that
support their predictions and connect them across inputs. This system-level
view echoes Quine's epistemological holism, which treats revision in relation
to an interconnected system rather than an isolated statement
\cite{quine1951two}. We introduced the Lifted State Hypothesis to study these
states and ask which inputs change together when LLMs are revised. The
hypothesis explains type-level generalization as the reuse of computations
across computation-relative types. We formalized each lifted state through an
activation filter and an output representation, and explained non-monotonic
revision through shared activation across a subtype and its parent type. When
the same state remains active for both, revising the subtype may also change
predictions for parent-type members whose computation remains valid. 
Experiments on the NMR-Type Dataset show that, across full fine-tuning,
LoRA, and in-context learning, models often learn the revised subtype while
also altering parent-type predictions that should remain unchanged. 
These findings support a generalization--revision duality: the same
state sharing that enables type-level transfer can also determine which
predictions change together.

\clearpage 

%% file: appendix/1_lifted_state_comparison.tex
\clearpage 
\section{Lifted States and Related Frameworks}
\label{appendix:1_lifted_states}

This section develops the two computational precedents underlying the
Lifted State Hypothesis: automata theory and lifted probabilistic inference.
Both frameworks replace repeated instance-level computation with a shared
state or factor defined over an equivalence class. They also make explicit
when such sharing remains valid and when additional distinctions must be
introduced. These frameworks serve as functional precedents rather than
architectural models of LLMs: automata and lifted models represent their
groups explicitly, whereas an LLM may realize corresponding computations
through  distributed latent states.

\subsection{Details on Automata}
\label{appendix:automata}

Let \(L\subseteq\Sigma^{*}\) be a language. For a prefix \(x\), its residual
language is
\[
L_x=\{w\in\Sigma^{*}:xw\in L\},
\]
which contains all continuations that lead from \(x\) to acceptance. Two
prefixes are equivalent when they induce the same acceptance decision under
every possible continuation:
\begin{equation}
x\equiv_L x'
~\Longleftrightarrow~
\mathbb{I}[xw\in L]
=
\mathbb{I}[x'w\in L],
~
\forall w\in\Sigma^{*}.
\label{eq:appendix_mn_equivalence}
\end{equation}
The Myhill--Nerode theorem characterizes when the resulting quotient is
finite \cite{hopcroft2001introduction}.

\begin{theorem}[Myhill--Nerode]
A language \(L\) is regular if and only if \(\equiv_L\) has finitely many
equivalence classes. In this case,
\[
Q=\Sigma^{*}/{\equiv_L},
\qquad
q_0=[\epsilon],
\qquad
F=\{[x]:x\in L\},
\]
together with \(\delta([x],a)=[xa]\), defines the unique minimal DFA
recognizing \(L\), up to isomorphism.
\end{theorem}

Each state of the minimal automaton therefore represents a class of prefixes
that are indistinguishable for the future acceptance computation. Once two
prefixes reach the same state, their surface forms and previous transition
paths no longer matter for that computation.

\paragraph{Example: parity computation.}
Consider the language
\[
L_{\mathrm{even}}
=
\left\{
x\in\{0,1\}^{*}
:
\#_{1}(x)\equiv 0 \pmod 2
\right\},
\]
where \(\#_{1}(x)\) is the number of \(1\)s in \(x\). Its minimal DFA has
two states: \(q_{\mathrm E}\), representing even parity, and
\(q_{\mathrm O}\), representing odd parity. The initial and accepting state
is \(q_{\mathrm E}\). Reading \(0\) preserves the current state, whereas
reading \(1\) switches between the two states.

For example, the input \(x=10110\) is processed as
\[
q_{\mathrm E}
\xrightarrow{1}q_{\mathrm O}
\xrightarrow{0}q_{\mathrm O}
\xrightarrow{1}q_{\mathrm E}
\xrightarrow{1}q_{\mathrm O}
\xrightarrow{0}q_{\mathrm O}.
\]
The automaton therefore ends in the non-accepting state
\(q_{\mathrm O}\), meaning that \(10110\) contains an odd number of
\(1\)s:
\[
10110\notin L_{\mathrm{even}}.
\]
Appending \(w=1\) changes the parity and returns the automaton to the
accepting state:
\[
q_{\mathrm O}\xrightarrow{1}q_{\mathrm E},
\qquad
101101\in L_{\mathrm{even}}.
\]

The shorter prefix \(x'=1\) also reaches \(q_{\mathrm O}\). Consequently,
the two prefixes produce the same acceptance result after any continuation:
\[
10110w\in L_{\mathrm{even}}
\quad\Longleftrightarrow\quad
1w\in L_{\mathrm{even}},
\qquad
\forall w\in\Sigma^{*}.
\]
Therefore, \(10110\equiv_{L_{\mathrm{even}}}1\). Their shared state retains
only the parity information required to determine acceptance after future
input symbols.

\paragraph{Transformers and finite-state computation.} Prior work suggests a functional rather than structural correspondence between Transformers and finite automata. Fixed-depth self-attention has limitations on some regular languages, while Transformers can learn or constructively simulate finite-state computations under suitable architectural and length-dependent conditions \cite{hahn-2020-theoretical,bhattamishra-etal-2020-ability, liu2023transformers}. Automata therefore provide a precedent for computation-relative state reuse, not a claim that Transformers are literally DFAs.

\paragraph{Implications for LLMs.}
The automata suggests several implications for
computation-relative states in LLMs.

\begin{enumerate}
    \item \textbf{Existence of computational states.}
    Automata show that distinct inputs can share a state that retains the
    information required for a target computation. By analogy, an LLM may
    realize such computational states through neurons, features, circuits, or
    distributed combinations of components.

    \item \textbf{Target-relative organization of states.}
    The distinctions preserved by a state are determined by the target
    computation and its output behavior. Computational states are therefore
    not intrinsic categories of inputs, but task-relative organizations formed
    according to which differences matter for producing the target output.

    \item \textbf{Revision through input--state reorganization.}
    When a subset requires different behavior, revision must change its routing
    or state membership. Unlike in automata, these relations are continuous,
    distributed, and implicit in LLMs.
    
\end{enumerate}

The correspondence is functional rather than structural. Automata partition
inputs into exact finite states. LLMs may realize comparable states in
continuous and distributed forms that vary with context and target computation.


\subsection{Details on Lifted Probabilistic Inference}
\label{appendix:lifted_inference}

\paragraph{Overview.}
We first describe ordinary probabilistic inference at the level of individual
random variables and then explain how lifted inference reuses repeated parts
of that computation. The overall procedure is as follows:
\begin{enumerate}
    \item \textbf{Construct the probabilistic model.}
    Define random variables and local factors whose product represents their
    joint distribution.

    \item \textbf{Form the ground-level computation.}
    Replace each logical variable with a concrete domain individual. The
    resulting object-specific variables and factors are called
    \emph{ground} variables and factors
    (e.g., replacing \(B\) with \(penguin\) gives
    \(\operatorname{CanFly}(penguin)\)).

    \item \textbf{Compute the target probability.}
    Condition on the observed evidence and marginalize variables that are not
    part of the query.

    \item \textbf{Lift repeated computation.}
    Identify ground factors that are interchangeable under the current
    evidence and query. Evaluate their shared computation collectively rather
    than processing each grounding separately.

    \item \textbf{Refine when necessary.}
    If new evidence distinguishes members of a lifted group, split or shatter
    the group before continuing inference.
\end{enumerate}

A probabilistic model is conventionally written using generic random-variable
names such as \(P(X,Y,Z)\). To distinguish the roles played by variables in
the computation, we instead write
\[
P(X,C,F),
\]
where \(X\) describes the instance, \(C\) represents the current context or
observed evidence, and \(F\) contains the variables queried by the target
computation. The corresponding inference problem is
\[
P(F\mid X=x,C=c).
\]
This notation introduces no additional probabilistic assumption. It only
organizes variables according to their roles in the selected computation.
The same variable may therefore serve as evidence in one query and as a
target variable in another.

\paragraph{Factors and parfactors.}
A ground random variable is fully instantiated and contains no remaining
logical variables. Let
\(\mathbf V=(V_1,\ldots,V_m)\) be a tuple of ground random variables.
For each \(V_i\), \(\operatorname{Val}(V_i)\) denotes the set of values that
the variable can take. Their joint assignment space is
\[
\operatorname{Val}(\mathbf V)
=
\operatorname{Val}(V_1)\times\cdots\times\operatorname{Val}(V_m).
\]
A ground factor is a non-negative function over this space:
\[
\psi_r:
\operatorname{Val}(\mathbf V_r)
\rightarrow
\mathbb R_{\geq 0}.
\]
For an assignment \(\omega=(x,c,f)\), a factorized probabilistic model has
the form
\begin{equation}
P(x,c,f)
=
\frac{1}{Z}
\prod_r
\psi_r\!\left(\omega_{\mathbf V_r}\right),
\label{eq:appendix_ground_factorization}
\end{equation}
where \(\mathbf V_r\) is the scope of factor \(\psi_r\),
\(\omega_{\mathbf V_r}\) is the assignment restricted to that scope, and
\(Z\) is the normalizing constant. A ground factor thus describes one
fully instantiated local computation.

A parfactor compactly represents a family of ground factors through a
parameterized template:
\[
g
=
\left\langle
\mathbf A_g,\phi_g,\kappa_g
\right\rangle.
\]
Here, \(\mathbf A_g\) is a tuple of parameterized random variables,
\(\phi_g\) is their shared potential, and \(\kappa_g\) constrains the
admissible groundings of their logical variables. A substitution \(\theta\)
satisfying \(\theta\models\kappa_g\) replaces the logical variables with
domain constants and produces the ground scope \(\mathbf A_g\theta\).
A parfactor model \(\mathcal G\) therefore defines
\begin{equation}
P(x,c,f)
=
\frac{1}{Z}
\prod_{g\in\mathcal G}
\prod_{\theta\models\kappa_g}
\phi_g\!\left(
\omega_{\mathbf A_g\theta}
\right).
\label{eq:appendix_parfactor_factorization}
\end{equation}
A ground factor specifies one instantiated computation, whereas a parfactor
reuses one parameterized computation across all groundings permitted by its
constraint.

This shared template does not by itself guarantee that every grounding can
be processed together. Lifted inference groups groundings only when they
remain interchangeable under the current evidence, constraints, and query.
It then evaluates their common factor once and retains the number of
groundings represented by that computation.

\begin{proposition}[Lifted factor reuse]
\label{prop:appendix_lifted_factor_reuse}
Fix \(X=x\) and \(C=c\). Suppose the target-dependent factor product can be
partitioned into lifted groups \(\mathcal T\), such that every grounding in
\(\tau\in\mathcal T\) contributes the same factor
\(\psi_\tau(f;x,c)\). Let \(n_\tau=|\tau|\). Then
\begin{equation}
P(F=f\mid X=x,C=c)
=
\frac{
\displaystyle
\prod_{\tau\in\mathcal T}
\psi_\tau(f;x,c)^{n_\tau}
}{
\displaystyle
\sum_{f'\in\operatorname{Val}(F)}
\prod_{\tau\in\mathcal T}
\psi_\tau(f';x,c)^{n_\tau}
}.
\label{eq:appendix_lifted_probability}
\end{equation}
\end{proposition}

The proposition follows by collecting identical ground factors and
normalizing over the possible assignments of \(F\). The pair
\((\psi_\tau,n_\tau)\) is therefore sufficient for the repeated contribution
of group \(\tau\).

\paragraph{Example: lifted computation over birds.}
Let
\[
\phi_{\mathrm{bird}}(X_B,C_B,F_B)
\]
be a potential shared across birds, where \(X_B\) represents the bird type,
\(C_B\) represents whether the bird has wings, and \(F_B\) represents its
flying and swimming behavior.

Consider three penguins and three eagles:
\[
\mathcal P=\{p_A,p_B,p_C\},
~
\mathcal E=\{e_A,e_B,e_C\},
~
\mathcal B=\mathcal P\cup\mathcal E.
\]
For each bird \(b\in\mathcal B\), define
\[
X_b
=
\bigl(
\operatorname{Penguin}(b),
\operatorname{Eagle}(b)
\bigr),
\qquad
C_b=\operatorname{HasWing}(b),
\]
and
\[
F_b
=
\bigl(
\operatorname{CanFly}(b),
\operatorname{CanSwim}(b)
\bigr).
\]
These variables form the parfactor
\[
g_{\mathrm{bird}}
=
\left\langle
\phi_{\mathrm{bird}},
(X_B,C_B,F_B),
B\in\mathcal B
\right\rangle.
\]
Substituting each bird for \(B\) produces six ground factors that share the
same potential \(\phi_{\mathrm{bird}}\).

Assume that every bird has wings. Penguins are non-flying swimmers, whereas
eagles are flying non-swimmers:
\[
\begin{array}{c|c|c|c}
 & X_b & C_b & F_b\\
\hline
b\in\mathcal P & (1,0) & 1 & (0,1)\\
b\in\mathcal E & (0,1) & 1 & (1,0)
\end{array}
\]
Thus, every penguin contributes the same factor value
\[
\phi_{\mathrm{bird}}\bigl((1,0),1,(0,1)\bigr),
\]
and every eagle contributes
\[
\phi_{\mathrm{bird}}\bigl((0,1),1,(1,0)\bigr).
\]
The six birds can therefore be represented by two lifted groups:
\[
\tau_{\mathrm P}=\{p_A,p_B,p_C\},
\qquad
\tau_{\mathrm E}=\{e_A,e_B,e_C\}.
\]
Their six ground-factor contributions reduce to
\begin{equation}
\phi_{\mathrm{bird}}\bigl((1,0),1,(0,1)\bigr)^{3}
\phi_{\mathrm{bird}}\bigl((0,1),1,(1,0)\bigr)^{3}.
\label{eq:appendix_bird_lifted_factors}
\end{equation}
The shared potential specifies the common computation, while each exponent
records the number of corresponding ground factors.

Now extend the context with
\(\operatorname{Injured}(b)\), and suppose that
\[
\operatorname{Injured}(p_A)=1,
\qquad
\operatorname{Injured}(p_B)
=
\operatorname{Injured}(p_C)
=
0.
\]
For computations in which injury affects the target variables, Penguin A is
no longer interchangeable with Penguins B and C. The penguin group must be
shattered into
\[
\tau_{\mathrm P}
\longrightarrow
\begin{cases}
\tau_{\mathrm{P,I}}=\{p_A\},\\
\tau_{\mathrm{P,U}}=\{p_B,p_C\}.
\end{cases}
\]
If \(\psi_{\mathrm{P,I}}\) and \(\psi_{\mathrm{P,U}}\) denote the resulting
injured and uninjured potentials, the former contribution
\[
\psi_{\mathrm P}(u)^3
\]
is replaced by
\[
\psi_{\mathrm{P,I}}(u)
\psi_{\mathrm{P,U}}(u)^2.
\]
The eagle group remains unchanged. New evidence has thus refined the previous
computational equivalence while preserving lifted computation among the
unaffected penguins.
\paragraph{Revision of lifted computation.}
A lifted probabilistic model explicitly represents both a shared computation
and its assigned ground instances. Let
\[
    g=\langle \mathbf A,\phi,\kappa\rangle,
\]
where \(\kappa\) specifies the groundings sharing \(\phi\). Revision can
proceed in two ways:

\begin{itemize}
    \item \textbf{State-preserving remapping.}
    The model retains \(\phi\) but changes the sample--factor mapping. New
    evidence may remove a grounding from one group or assign it to another,
    changing which instances reuse each computation.

    \item \textbf{Lifted-state reconstruction.}
    The model refines the shared organization as
    \[
        \langle \mathbf A,\phi,\kappa\rangle
        \longrightarrow
        \left\{
            \langle \mathbf A_r,\phi_r,\kappa_r\rangle
        \right\}_{r=1}^{R},
        \qquad
        \bigsqcup_{r=1}^{R}\kappa_r=\kappa.
    \]
    This reconstruction may involve:
    \begin{itemize}
        \item \emph{Sample splitting:}
        partitioning or shattering previously interchangeable groundings
        \cite{braz2005lifted,milch2008lifted,pmlr-v22-taghipour12},
        initially without changing their potential.

        \item \emph{Context reconstruction:}
        changing the factor scope or adding an evidence-dependent factor so
        that a previously omitted distinction affects the computation.

        \item \emph{Value revision:}
        replacing \(\phi\) with group-specific potentials \(\phi_r\).
    \end{itemize}
\end{itemize}

These operations may be combined by splitting a group, adding a contextual
dependency, and assigning the distinguished subgroup a different potential.

From the perspective of lifted inference, the closest analogue of revision is
a symmetry-preserving refinement. Poole treats distinguished individuals
separately while retaining collective reasoning over the remaining population,
and FOVE separates groundings only when current evidence or query conditions
invalidate their interchangeability
\cite{poole2003first,braz2005lifted}. Splitting and shattering therefore
expose distinctions required for exact inference without discarding valid
sharing among unaffected instances.

Later methods extend this principle by preserving lifting within increasingly
expressive intermediate representations. Counting formulas retain
interchangeability within potentials, while richer constraint languages
represent more precise groups without reducing them to singleton groundings
\cite{milch2008lifted,pmlr-v22-taghipour12}. Thus, refinement should preserve
the largest groups that remain interchangeable under the revised evidence and
target computation; full grounding is required only when no valid lifted
organization remains.

\paragraph{Information required for selective revision.}
Both forms of revision depend on the previous mapping between ground instances
and shared computations. Remapping requires knowing which instances should
leave or enter an existing group. Splitting requires identifying both the
subgroup to revise and the remaining members whose computation should be
preserved. Changing a potential requires knowing its full grounding scope,
because every instance still assigned to that potential receives the revised
contribution. Lifted models retain this information explicitly through
constraints, factor scopes, grounding sets, and counts. Selective revision is
therefore not determined by the new evidence alone; it also depends on the
existing organization of samples.

\paragraph{Implications for LLMs.}
The same distinction suggests two broad revision mechanisms in LLMs. A model
may preserve a lifted state and alter which episodes activate it by changing
their representations, filters, or routing. Alternatively, it may reconstruct
the lifted organization by separating a subtype, composing the computation
with additional context, or changing the effective value \(\phi_k\).

Unlike an explicit lifted model, however, an LLM does not normally maintain an
addressable record of which earlier episodes formed. Data-attribution methods can estimate the training samples that
influence a prediction, while continual-learning methods use parameter
constraints or episodic replay to preserve earlier behavior
\cite{pmlr-v70-koh17a,kirkpatrick2017overcoming,lopez2017gradient}.
These methods provide partial access to prior learning but do not recover an
explicit episode--state membership relation.

This limitation is also reflected in model editing. Local editing methods
seek to modify a selected association while preserving unrelated behavior
\cite{mitchell2022fast,meng2022locating}, yet edits may fail to propagate to
related predictions or spread to instances that should remain unchanged
\cite{cohen2024ripple,qin2024messy}. Rule-level editing similarly emphasizes
that revision requires identifying the broader set of instances governed by
the edited computation
\cite{zhou2025ruleedit}. Replay can partially recover this scope by
re-exposing previous samples, whereas updates without replay must infer it
indirectly from current parameters and activations.


\begin{figure*}[t!]
    \centering
    \includegraphics[width=1.0\linewidth]{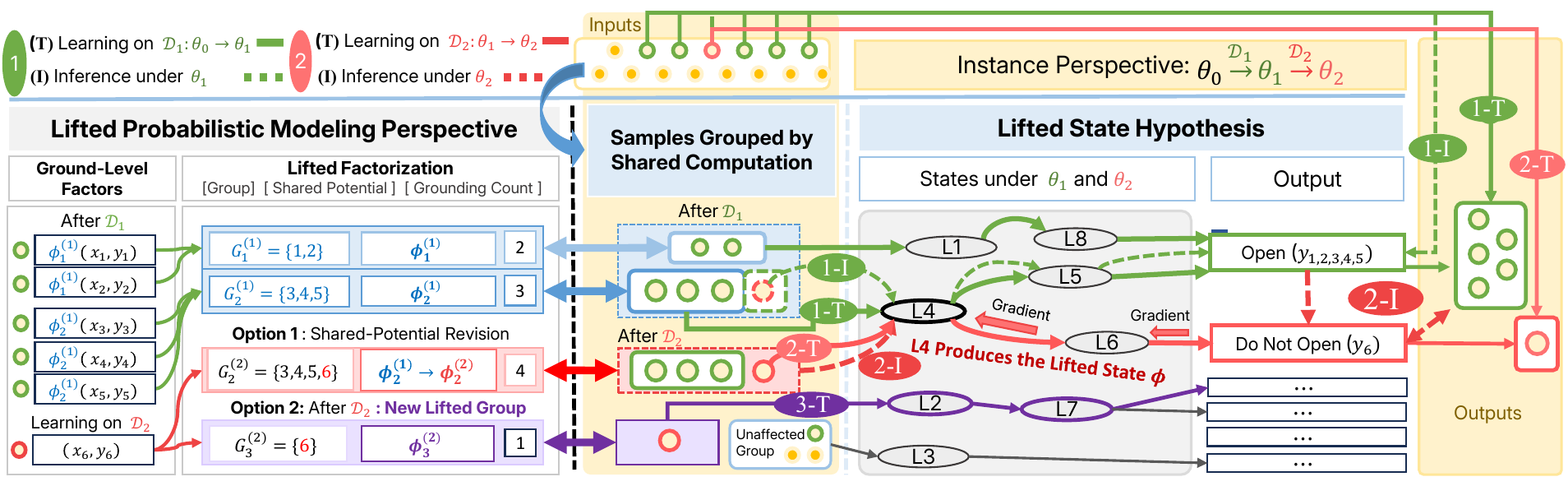}
    \caption{
Functional correspondence between lifted probabilistic modeling and the
Lifted State Hypothesis under sequential learning.
\(1\text{-T}\) denotes learning on \(\mathcal D_1\),
\(\theta_0\!\rightarrow\!\theta_1\), and \(1\text{-I}\) denotes inference
under \(\theta_1\). Likewise, \(2\text{-T}\) denotes learning on
\(\mathcal D_2\), \(\theta_1\!\rightarrow\!\theta_2\), and \(2\text{-I}\)
denotes inference under \(\theta_2\). Learning \(\mathcal D_2\) may revise a
shared computation and affect its existing group, or form a new group for
\(x_6\) while preserving the others. Solid and dashed arrows indicate
learning and inference, respectively.
}
    \label{fig:main_fig_overview}
\end{figure*}
\input{appendix/table-comparison}

\subsection{Detailed Comparison with LLMs}
\label{appendix}

Figure~\ref{fig:main_fig_overview} illustrates the central learning dynamics
of the Lifted State Hypothesis, while
Table~\ref{tab:type-level-inference} compares its computational organization
with automata and lifted probabilistic inference. The comparison is organized
through six questions---who, where, when, what, how, and why---that distinguish
the cases being grouped, the state representation, the context of
equivalence, the criterion for grouping, the reuse mechanism, and the target
computation. We first describe the vector realization and then examine how
shared computation changes under sequential learning.

\paragraph{Vector realization.}
An episode (e=(c,x,f)) is associated with an episode-conditioned vector
\(q_\theta(e)\). A lifted state represents a reusable component of this vector,
rather than its complete representation. We write
\[
s_k=(z_k,\phi_k),
\]
where the filter \(z_k\) determines whether and how strongly an episode
recruits the state, and the value \(\phi_k\) specifies its contribution to the
target computation. Multiple episodes may therefore reuse the same value with
different activation strengths while retaining additional episode-specific
information.

This realization differs from a finite automaton state, which completely
summarizes the information from a prefix that is relevant to future
acceptance. A lifted state in an LLM need not exhaust the model's current
representation or uniquely determine its output. It is one component within a
larger vector computation and may coexist with other shared or
episode-specific components.

The relation to lifted probabilistic inference is closer at the level of
reuse. A parfactor assigns one potential to multiple groundings that are
equivalent under the selected relational structure. Similarly, a lifted state
provides one shared contribution to multiple episodes belonging to a
computation-relative type. The difference is that parfactor membership and
potential reuse are explicitly defined by a probabilistic model, whereas the
corresponding grouping in an LLM is latent, graded, and learned from data.

\paragraph{Sequential learning and revision.}
After learning \(\mathcal D_1\), an unobserved subtype may be routed through
the same latent state as its parent type and therefore inherit the parent
output during inference. This is the vector analogue of assigning multiple
cases to the same automaton state or evaluating multiple groundings through
the same parfactor potential.

When \(\mathcal D_2\) later provides subtype-specific supervision, the model
must revise this organization. If the subtype and preserved parent episodes
continue to recruit the same state, changing its filter or value may move
their computations together. From the lifted probabilistic perspective, this
resembles modifying a shared potential while leaving the corresponding
groundings in one lifted group.

Localized revision instead requires an effective distinction. The model may
reroute the subtype, form a new state, or introduce a compensating
episode-specific computation. This corresponds most closely to refining an
automaton state partition or splitting a lifted probabilistic group so that
the revised cases no longer share the same computation with the preserved
cases.

\paragraph{Comparison across frameworks.}
All three frameworks group cases that are equivalent relative to a selected
future computation and reuse a common state or computation across those cases.
They differ, however, in what is represented and how explicitly the grouping
is defined. Automata use discrete states and exact future-language
equivalence. Lifted probabilistic inference uses symbolic variables,
parfactors, and grounding counts. The Lifted State Hypothesis instead concerns
latent filter--value components embedded in a distributed vector
representation. Its main claim is therefore not merely that LLM inputs form
representational clusters, but that multiple episodes may functionally depend
on a common computational component whose revision can coordinate their
behavior.

%% file: appendix/table-comparison.tex

\newcolumntype{L}[1]{%
  >{\raggedright\arraybackslash
    \hsize=#1\hsize
    \linewidth=\hsize}X%
}

\newcommand{\dimhead}[2]{%
  {\footnotesize
    \textbf{#1}\par
    \vspace{-1pt}
    \textit{#2}%
  }%
}

\begin{table*}[t!]
\centering
\caption{Comparison of type-level computation across automata, lifted
probabilistic inference, and language models. The table characterizes the
common phenomenon of sharing computation across recurring cases that are
treated as equivalent for a target. The Who--Why dimensions describe how each
framework identifies these cases and realizes computational reuse.}
\label{tab:type-level-inference}

\setlength{\tabcolsep}{3.5pt}
\renewcommand{\arraystretch}{1.10}
\setlength{\aboverulesep}{0pt}
\setlength{\belowrulesep}{2.5pt}

\small

\begin{tabularx}{\textwidth}{
    @{}
    L{0.65}
    L{1.00}
    L{1.05}
    L{1.25}
    L{1.20}
    L{0.85}
    @{}
}
\toprule

\dimhead{Framework}{Who}
&
\dimhead{State realization}{Where}
&
\dimhead{Equivalence scope}{When}
&
\dimhead{Type criterion}{What}
&
\dimhead{Reuse mechanism}{How}
&
\dimhead{Computation}{Why}
\\

\midrule

\textbf{Automaton}
&
Finite automaton state
&
All task-relevant continuations
&
Prefixes with identical future acceptance behavior
&
Map each equivalence class to one state
&
Decide acceptance
\par
\(f_L(x;w)\)
\\

\addlinespace[3pt]

\textbf{Lifted probabilistic model}
&
Parfactor and grounding count
&
Fixed evidence, constraints, and query
&
Interchangeable groundings with identical factor contributions
&
Evaluate one shared potential and reuse it by count
&
Compute
\par
\(P(F\mid X,C)\)
\\

\addlinespace[3pt]

\textbf{Language model}
&
Latent filter--value component
&
Fixed \(\theta\), context scope \(\Gamma\), and computation \(f\)
&
Samples indistinguishable in observed \(f\)-relevant behavior
&
Reuse \(\phi_k\) through episode-dependent activation \(a_k(e)\)
&
Support
\par
\(f_\theta(c,x)\)
\\

\bottomrule
\end{tabularx}
\end{table*}

%% file: appendix/2_theory.tex
\section{Neural Basis of the Lifted State Hypothesis}
\label{appendix:lifted_state_details}

The main text represents a lifted state through a minimal filter--value
decomposition. This abstraction is motivated by the vector-valued structure
of Transformer computation and by evidence that recurring tasks, functions,
and belief states can be associated with directions or low-dimensional
geometry in activation space
\cite{hendel2023taskvectors,todd2024functionvectors,shai2024beliefstate}.
It does not require a lifted state to correspond to one neuron, one sparse
feature, or a uniquely identifiable vector.

\subsection{Filter--Value Structure}
\label{appendix:filter_value_basis}

Let \(h\in\mathbb{R}^{d}\) enter a standard Transformer feed-forward block.
If \(k_j^\top\) is the \(j\)-th row of \(W_{\mathrm{in}}\) and \(v_j\) is the
\(j\)-th column of \(W_{\mathrm{out}}\), then
\begin{equation}
    \operatorname{MLP}(h)
    =
    \sum_{j=1}^{d_{\mathrm{ff}}}
    \sigma\!\left(k_j^\top h\right)v_j.
\label{eq:appendix_mlp_filter_value}
\end{equation}
Each term combines an input-dependent recruitment coefficient with a shared
output vector. This structure motivates key--value interpretations of
Transformer MLPs \cite{geva-etal-2021-transformer}; gated MLPs retain the same
organization with more general coefficients.

A recurring computation need not align with a single architectural neuron.
The joint contribution of several neurons, features, heads, or layers may be
locally summarized as
\begin{equation}
    \sum_{j\in\mathcal{J}_k}
    \alpha_j\!\left(h(e)\right)v_j
    \approx
    a_k(e)\phi_k.
\label{eq:appendix_aggregate_state}
\end{equation}
Here, \(a_k(e)\) summarizes recruitment and \(\phi_k\) summarizes the effective
vector contribution. The approximation may represent a neuron, direction,
subspace, or distributed circuit without assuming that its underlying
implementation is one-dimensional.

\subsection{Candidate Neural Bases}
\label{appendix:state_reuse_basis}

Knowledge- and skill-neuron studies show that interventions on individual or
sparse neuron sets can affect model behavior
\cite{dai-etal-2022-knowledge,wang-etal-2022-finding-skill}. Under
superposition, however, meaningful features may be distributed across
non-orthogonal directions rather than aligned with individual neurons
\cite{elhage2022toy}. The apparent granularity of a candidate state may
therefore depend on the basis and interface at which the model is analyzed.
A distributed computation may appear as several components internally while
producing a compact contribution at a later residual-stream location.

Sparse autoencoders provide another candidate basis by separating activation
patterns into sparse coefficients and decoder directions
\cite{bricken2023monosemanticity,templeton2024scaling}. Residual-stream
directions, subspaces, and multi-layer circuits provide further possibilities
\cite{olah2020zoom,todd2024functionvectors}. These representations are
candidate implementations only: common activation does not by itself
establish a lifted state.

\subsection{Type-Conditioned Evidence for a Lifted State}
\label{appendix:evidence_for_lifted_state}

Evidence for a lifted state must be conditioned on an independently specified
computation-relative type. Episodes should first be sampled from the proposed
type, relevant subtypes, and appropriate controls without using the candidate
neural component to define the grouping. This separation prevents the type
from being defined retrospectively by whichever activation pattern happens to
be observed.

The candidate component should then be recruited systematically across type
members and should provide a common contribution to the target computation.
For an SAE feature, for example, common activation should be accompanied by a
shared reconstructed contribution through its decoder direction. Its
recruitment and contribution should also distinguish the proposed type from
controls that are similar in surface form but differ in the target
computation.

Finally, ablation, patching, suppression, or steering should alter the same
target computation across the grouped episodes. Coordinated behavioral
revision should correspond to changes in the component's recruitment or
contribution, whereas successful subtype separation should involve reduced
sharing, rerouting, or compensating computation elsewhere.

Thus, neurons, sparse features, directions, subspaces, and circuits provide
possible neural bases. The Lifted State Hypothesis imposes the stronger
requirement that a predefined episode type systematically and causally reuse
the same functional contribution.


\section{Proofs and Derivations for Non-Monotonic Revision}
\label{appendix:theoretical-analysis}

This appendix restates the formal claims from the main analysis in
self-contained form, provides their proofs, and derives the first-order
approximation used to characterize state revision. Throughout this section,
notation associated with the episode representation
\(\textcolor{routingblue}{q}\), activation filter
\(\textcolor{filterorange}{z}\), and state value
\(\textcolor{valuepurple}{\phi}\) is color-coded accordingly. 

\subsection{Subtype Activation through Filter Alignment}
\label{appendix:subtype_activation_transfer}

\paragraph{Proposition~2
(Aligned filter updates increase subtype activation; restated).}
Let \(e'\) be an unseen episode belonging to subtype \(\tau'\). For
\(t\in\{0,1\}\), let its activation under state \(s_k\) be
\begin{equation}
    a_{k,t}(e')
    =
    \sigma\!\left(
        \textcolor{filterorange}{z_{k,t}}^{\top}
        \textcolor{routingblue}{q(e')}
    \right),
\end{equation}
where the episode representation \(q(e')\) is held fixed during the initial
learning stage. The shared filter moves from \(z_{k,0}\) to \(z_{k,1}\), with
update direction
\begin{equation}
    \textcolor{filterorange}{\Delta_1 z_k}
    :=
    \textcolor{filterorange}{
        z_{k,1}-z_{k,0}
    }.
\end{equation}
If the unseen subtype representation is positively aligned with this filter
update,
\begin{equation}
    \textcolor{routingblue}{q(e')}^\top
    \textcolor{filterorange}{\Delta_1 z_k}
    >
    0,
\label{eq:appendix_subtype_alignment}
\end{equation}
then its activation increases:
\begin{equation}
    a_{k,1}(e')
    >
    a_{k,0}(e').
\end{equation}

\begin{proof}
Let
\begin{equation}
    \ell_{k,t}(e')
    :=
    \textcolor{filterorange}{z_{k,t}}^{\top}
    \textcolor{routingblue}{q(e')}
\end{equation}
denote the activation logit of \(s_k\) for \(e'\). Since
\begin{equation}
    \textcolor{filterorange}{z_{k,1}}
    =
    \textcolor{filterorange}{z_{k,0}}
    +
    \textcolor{filterorange}{\Delta_1 z_k},
\end{equation}
and \(q(e')\) is fixed, the updated logit satisfies
\begin{align}
    \ell_{k,1}(e')
    &=
    \left(
        \textcolor{filterorange}{z_{k,0}}
        +
        \textcolor{filterorange}{\Delta_1 z_k}
    \right)^\top
    \textcolor{routingblue}{q(e')}
    \nonumber\\
    &=
    \textcolor{filterorange}{z_{k,0}}^\top
    \textcolor{routingblue}{q(e')}
    \nonumber\\
    &\quad+
    \textcolor{routingblue}{q(e')}^\top
    \textcolor{filterorange}{\Delta_1 z_k}
    \nonumber\\
    &=
    \ell_{k,0}(e')
    +
    \textcolor{routingblue}{q(e')}^\top
    \textcolor{filterorange}{\Delta_1 z_k}.
\end{align}
By the positive-alignment condition in
Equation~\eqref{eq:appendix_subtype_alignment},
\begin{equation}
    \ell_{k,1}(e')
    >
    \ell_{k,0}(e').
\end{equation}
Because \(\sigma\) is strictly increasing,
\begin{align}
    a_{k,1}(e')
    &=
    \sigma\!\left(
        \ell_{k,1}(e')
    \right)
    \nonumber\\
    &>
    \sigma\!\left(
        \ell_{k,0}(e')
    \right)
    \nonumber\\
    &=
    a_{k,0}(e').
\end{align}
\end{proof}

The proposition makes the transfer mechanism explicit. Although \(e'\) is
unseen during the initial learning stage, its activation increases when the
shared-filter update points toward its episode representation. Thus,
activation can transfer to an unseen subtype when its representation is
positively aligned with the learned filter direction. The condition concerns
the direction of the filter update, rather than the absolute orientation of
either \(z_{k,0}\) or \(z_{k,1}\).

The fixed-representation assumption isolates activation transfer induced by
the filter update. If \(q(e')\) also changes, the activation-logit change
additionally includes a representation-change term and an interaction between
the representation and filter changes.

\subsection{Persistent Activation Sharing Bounds Differential Revision}
\label{appendix:persistent_sharing}

\paragraph{Proposition~3
(Persistent sharing bounds differential revision; restated).}
Let \(e\) be a preserved parent episode and \(e'\) a revised subtype episode.
For \(\bar e\in\{e,e'\}\) and \(t\in\{1,2\}\), define
\begin{equation}
\begin{aligned}
    v_{k,t}(\bar e)
    &:=
    a_{k,t}(\bar e)
    \textcolor{valuepurple}{\phi_{k,t}},\\
    \Delta v_k(\bar e)
    &:=
    v_{k,2}(\bar e)-v_{k,1}(\bar e).
\end{aligned}
\label{eq:appendix_state_changes}
\end{equation}
Suppose that their activation discrepancy remains bounded at both stages:
\begin{equation}
\begin{aligned}
    \left|
        a_{k,t}(e)-a_{k,t}(e')
    \right|
    &\leq
    \epsilon_{k,t},\\
    &\hspace{-1.5em}t\in\{1,2\}.
\end{aligned}
\label{eq:appendix_activation_discrepancy}
\end{equation}
Then the difference between their state-specific revisions satisfies
\begin{equation}
\begin{aligned}
    &\left\|
        \Delta v_k(e)-\Delta v_k(e')
    \right\|\\
    &\quad\leq
    \epsilon_{k,2}
    \left\|
        \textcolor{valuepurple}{\phi_{k,1}}
        +
        \textcolor{valuepurple}{\Delta\phi_k}
    \right\|\\
    &\qquad+
    \epsilon_{k,1}
    \left\|
        \textcolor{valuepurple}{\phi_{k,1}}
    \right\|.
\end{aligned}
\label{eq:appendix_differential_revision_bound}
\end{equation}

\begin{proof}
Using
\begin{equation}
    \textcolor{valuepurple}{\phi_{k,2}}
    =
    \textcolor{valuepurple}{\phi_{k,1}}
    +
    \textcolor{valuepurple}{\Delta\phi_k},
\end{equation}
Equation~\eqref{eq:appendix_state_changes} gives
\begin{align}
    &\Delta v_k(e)-\Delta v_k(e')
    \nonumber\\
    &=
    \left(
        a_{k,2}(e)-a_{k,2}(e')
    \right)
    \left(
        \textcolor{valuepurple}{\phi_{k,1}}
        +
        \textcolor{valuepurple}{\Delta\phi_k}
    \right)
    \nonumber\\
    &\quad-
    \left(
        a_{k,1}(e)-a_{k,1}(e')
    \right)
    \textcolor{valuepurple}{\phi_{k,1}}.
\end{align}
Taking norms and applying the triangle inequality yields
\begin{align}
    &\left\|
        \Delta v_k(e)-\Delta v_k(e')
    \right\|
    \nonumber\\
    &\leq
    \left|
        a_{k,2}(e)-a_{k,2}(e')
    \right|
    \left\|
        \textcolor{valuepurple}{\phi_{k,1}}
        +
        \textcolor{valuepurple}{\Delta\phi_k}
    \right\|
    \nonumber\\
    &\quad+
    \left|
        a_{k,1}(e)-a_{k,1}(e')
    \right|
    \left\|
        \textcolor{valuepurple}{\phi_{k,1}}
    \right\|.
\end{align}
Applying the bounds in
Equation~\eqref{eq:appendix_activation_discrepancy} gives
\begin{equation}
\begin{aligned}
    &\left\|
        \Delta v_k(e)-\Delta v_k(e')
    \right\|\\
    &\quad\leq
    \epsilon_{k,2}
    \left\|
        \textcolor{valuepurple}{\phi_{k,1}}
        +
        \textcolor{valuepurple}{\Delta\phi_k}
    \right\|\\
    &\qquad+
    \epsilon_{k,1}
    \left\|
        \textcolor{valuepurple}{\phi_{k,1}}
    \right\|.
\end{aligned}
\end{equation}
\end{proof}

In the exact-sharing case,
\begin{equation}
\begin{aligned}
    a_{k,t}(e)
    &=
    a_{k,t}(e'),\\
    &\hspace{-1.5em}t\in\{1,2\},
\end{aligned}
\end{equation}
we have \(\epsilon_{k,1}=\epsilon_{k,2}=0\), and therefore
\begin{equation}
    \Delta v_k(e)
    =
    \Delta v_k(e').
\end{equation}
Thus, while the two episodes retain the same activation, the shared state
cannot produce different revisions for the parent and subtype. Differential
revision therefore requires activation separation, rerouting, or compensating
computation elsewhere. This conclusion concerns the state-specific
contribution \(v_k\), not the complete representations or outputs.

%% file: appendix/3_nmr_type_dataset.tex
\section{NMR-Type Dataset Construction}
\label{appendix:3_datasets}
The \textsc{NMR-Type Dataset} instantiates overlapping logical types through
divisibility relations over integers. It first teaches a broad rule and then
introduces conflicting supervision for a subtype together with replayed
samples. This construction tests whether a model generalizes through a shared
type-level computation and can localize subsequent revision despite replay of
preserved cases.

\paragraph{Logical types and finite realizations.}
Let \(\mathcal{X}_{\mathrm{pool}}\subset\mathbb{Z}_{>0}\) be a finite
pool of candidate integers, and define
\[
    P_k(x)\coloneqq\mathbb{I}[k\mid x].
\]
The corresponding logical type is
\begin{equation}
    X_{\tau_k}
    =
    \left\{
        x\in\mathcal{X}_{\mathrm{pool}}
        :
        P_k(x)=1
    \right\}.
    \label{eq:nmr-logical-type}
\end{equation}
These logical types may overlap or stand in an inclusion relation. In
particular,
\begin{equation}
    P_4(x)\Rightarrow P_2(x),
    \qquad
    X_{\tau_4}\subseteq X_{\tau_2}.
    \label{eq:nmr-logical-subtype}
\end{equation}

The model, however, is not given these type formulas or their complete
denotations. It observes only finite labeled realizations. Consequently, a
finite set of positive even integers may be consistent both with a narrow
type excluding multiples of \(4\) and with the broader type \(P_2\).
Generalization to unobserved multiples of \(4\) therefore indicates that the
model has induced the broader even-number type from the finite realization.

The benchmark itself is generated from mutually disjoint annotated groups.
This partition is part of the underlying data construction but is not
directly provided to the model. The second stage does not create this
partition; rather, it provides evidence revealing a distinction that was not
identifiable from the first-stage observations alone.

\paragraph{Disjoint group construction.}
The benchmark retains the following four groups:
\begin{equation}
\begin{aligned}
    \mathcal{G}_2
    &= \{x\in\mathcal{X}_{\mathrm{pool}}
        : P_2(x)\land\neg P_3(x)\land\neg P_4(x)\land\neg P_5(x)\},\\
    \mathcal{G}_3
    &= \{x\in\mathcal{X}_{\mathrm{pool}}
        : P_3(x)\land\neg P_2(x)\land\neg P_5(x)\},\\
    \mathcal{G}_4
    &= \{x\in\mathcal{X}_{\mathrm{pool}}
        : P_4(x)\land\neg P_5(x)\},\\
    \mathcal{G}_5
    &= \{x\in\mathcal{X}_{\mathrm{pool}}
        : P_5(x)\land\neg P_2(x)\land\neg P_3(x)\}.
\end{aligned}
\label{eq:nmr-disjoint-groups}
\end{equation}
The retained benchmark support is
\begin{equation}
    \mathcal{X}^{\circ}
    =
    \mathcal{G}_2
    \mathbin{\dot{\cup}}
    \mathcal{G}_3
    \mathbin{\dot{\cup}}
    \mathcal{G}_4
    \mathbin{\dot{\cup}}
    \mathcal{G}_5.
    \label{eq:nmr-retained-support}
\end{equation}
Thus, the annotated groups are pairwise disjoint even though their primitive
logical predicates need not be. For example, \(P_3\) and \(P_4\) overlap
logically, and an integer divisible by both \(3\) and \(4\), but not by
\(5\), is assigned to \(\mathcal{G}_4\).

Restricted to integers from \(1\) to \(30\), the groups are
\begin{equation}
\begin{aligned}
    \mathcal{G}_2\cap\{1,\ldots,30\}
    &= \{2,14,22,26\},\\
    \mathcal{G}_3\cap\{1,\ldots,30\}
    &= \{3,9,21,27\},\\
    \mathcal{G}_4\cap\{1,\ldots,30\}
    &= \{4,8,12,16,24,28\},\\
    \mathcal{G}_5\cap\{1,\ldots,30\}
    &= \{5,25\}.
\end{aligned}
\label{eq:nmr-group-examples}
\end{equation}
For instance, \(12\) and \(24\) belong to \(\mathcal{G}_4\) despite also
being divisible by \(3\), whereas \(20\) is excluded because it is divisible
by both \(4\) and \(5\).

Within the retained support, the broad even-number type is realized as
\begin{equation}
    X_{\tau_2}\cap\mathcal{X}^{\circ}
    =
    \mathcal{G}_2
    \mathbin{\dot{\cup}}
    \mathcal{G}_4.
    \label{eq:nmr-even-partition}
\end{equation}
Accordingly, \(\mathcal{G}_4\) is the revised subtype of the even-number
rule, whereas \(\mathcal{G}_2\) is the retained non-\(P_4\) region.
Integers outside \(\mathcal{X}^{\circ}\), including unused or ambiguous
divisor combinations, are not included in the benchmark.

\paragraph{Stage \(W_1\).}
For each group, let \(S_{t,k}\subseteq\mathcal{G}_k\) denote its finite
support at stage \(W_t\). The first-stage dataset is
\begin{equation}
\begin{aligned}
    \mathcal{D}_{W_1}
    ={}&
    \left\{
        (x,\texttt{positive})
        :
        x\in S_{1,2}
    \right\}\\
    &\cup
    \left\{
        (x,\texttt{negative})
        :
        x\in S_{1,3}
    \right\}\\
    &\cup
    \left\{
        (x,\texttt{positive})
        :
        x\in S_{1,5}
    \right\},
\end{aligned}
\label{eq:nmr-w1-dataset}
\end{equation}
with no examples from \(\mathcal{G}_4\). Because
\[
    S_{1,2}
    \subseteq
    \mathcal{G}_2
    \subseteq
    X_{\tau_2},
\]
the observed positive even examples are consistent both with a rule
restricted to \(\mathcal{G}_2\) and with the broader rule
\[
    P_2(x)\mapsto\texttt{positive}.
\]
A positive prediction on \(\mathcal{G}_4\) therefore indicates that the
model has generalized the broader \(P_2\) rule beyond its observed finite
support.

\paragraph{Stage \(W_2\).}
Let \(R_k\subseteq S_{1,k}\) denote the examples replayed from
\(\mathcal{G}_k\), for \(k\in\{2,3,5\}\). The second-stage dataset is
\begin{equation}
\begin{aligned}
    \mathcal{D}_{W_2}
    ={}&
    \left\{
        (x,\texttt{negative})
        :
        x\in S_{2,4}
    \right\}\\
    &\cup
    \left\{
        (x,\texttt{positive})
        :
        x\in R_2
    \right\}\\
    &\cup
    \left\{
        (x,\texttt{negative})
        :
        x\in R_3
    \right\}\\
    &\cup
    \left\{
        (x,\texttt{positive})
        :
        x\in R_5
    \right\}.
\end{aligned}
\label{eq:nmr-w2-dataset}
\end{equation}
The negative \(\mathcal{G}_4\) examples reveal that the previously plausible
broad even-number rule must be refined. The replayed
\(\mathcal{G}_2\) examples specify where this revision must stop, while the
\(\mathcal{G}_3\) and \(\mathcal{G}_5\) examples provide retention signals
for the other annotated groups. Thus, \(W_2\) makes the previously latent
partition
\(\mathcal{G}_2\mathbin{\dot{\cup}}\mathcal{G}_4\) empirically
distinguishable.

\paragraph{Evaluation target.}
The required predictions after the two stages are
\begin{equation}
\begin{aligned}
    W_1:\quad&
    \mathcal{G}_4
    \mapsto
    \texttt{positive},
    \\
    W_2:\quad&
    \mathcal{G}_4
    \mapsto
    \texttt{negative},
    \qquad
    \mathcal{G}_2
    \mapsto
    \texttt{positive},
    \\
    &
    \mathcal{G}_3
    \mapsto
    \texttt{negative},
    \qquad
    \mathcal{G}_5
    \mapsto
    \texttt{positive}.
\end{aligned}
\label{eq:nmr-evaluation-targets}
\end{equation}
Therefore, \(\mathcal{G}_4\) measures initial type-level generalization
followed by subtype revision, whereas \(\mathcal{G}_2\) measures
preservation within the same initially plausible parent type.
The \(\mathcal{G}_3\) and \(\mathcal{G}_5\) groups measure retention outside
the even-number partition. The model observes only individual inputs and
labels; the logical formulas, group identities, and disjoint partition
remain latent benchmark annotations.

The modulo construction should therefore be viewed not as a proposal for
modulo reasoning itself, but as a controlled method for constructing
type-based flips. Such non-monotonic flips arise when a type-level fact
induced from incomplete evidence must later be selectively overturned as
additional observations reveal a finer type boundary. The task thus tests
revision of an initially plausible generalization rather than arbitrary
reversal of an isolated label.

\clearpage 

\section{Experimental Settings}
\label{appendix:experimental_settings}

\paragraph{Prompt format for finetuning.}
For finetuning, each example is formatted as a single integer-to-label query.
The labels are class names, positive and negative. The model is trained to generate only the assistant answer, while the
prompt tokens are masked out from the loss.
\begin{tcolorbox}[
    title={Finetuning prompt},
    breakable,
    colback=gray!1,
    colframe=black!80,
    boxrule=0.5pt,
    arc=1mm,
    left=1mm,
    right=1mm,
    top=1mm,
    bottom=1mm
]
\footnotesize
\ttfamily
\raggedright

Instruction:\par
Map the integer to positive, negative, or neutral
using the assigned modulo type rule.
These are arbitrary class labels, not the
mathematical sign of the integer.\par
\medskip

Respond with exactly one label wrapped in\par
<answer> and </answer> tags.\par
\medskip

Integer: \{x\}\par
\medskip

Question:\par
According to the assigned type rule,\par
what is this integer's label?\par
\medskip

Answer:\par
\medskip

Assistant completion:\par
<answer>\{positive|negative\}</answer>
\end{tcolorbox}
When a model tokenizer provides a chat template, we apply the model-specific
chat template to the user prompt and assistant answer. Otherwise, we use the
plain-text prompt above. During evaluation, decoding is deterministic
(\texttt{do\_sample=False}) with at most 16 new tokens.

\paragraph{Training schedules.}
We compare \(W_1+W_2\) and \(W_2\)-only schedules, each comprising 300
gradient updates. The former applies 150 updates to each stage, whereas the
latter applies all 300 updates to \(W_2\). Each stage contains 32 examples.
For \(W_1\), we sample
\[
    (T_2,T_3,T_5)=(8,16,8),
\]
and for \(W_2\),
\[
    (T_2,T_3,T_4,T_5)=(8,8,8,8).
\]
Both sets contain 16 \texttt{positive} and 16 \texttt{negative} examples.
For each type, examples are drawn by a seed-specific sampler from its first
64 training candidates. In \(W_2\), \(T_4\) is the revised subtype, while
\(T_2\), \(T_3\), and \(T_5\) serve as replay-based preservation targets.
\paragraph{Optimization.}
Full-parameter finetuning trains all model parameters. LoRA finetuning freezes
the base model and trains LoRA adapters only. We use AdamW unless otherwise
specified. The full-parameter learning rate is \(1\times10^{-5}\), and the
LoRA learning rate is \(5\times10^{-5}\). The learning rate is linearly decayed
over the full schedule to \(0.7\) times its initial value. We use weight decay
0.0, gradient clipping with maximum norm 1.0, and no warmup. Gradient
checkpointing is enabled when supported by the model.

The effective batch sizes are stage-specific. For \(W_1\), the batch size is 3,
matching the three \(W_1\) types. For \(W_2\), the batch size is 4, matching the
four \(W_2\) types. Gradient accumulation is 1 in both stages. We evaluate every
30 gradient updates and at the end of each stage.

\paragraph{LoRA configuration.}
For LoRA experiments, we use rank \(r=16\), scaling parameter
\(\alpha=32\), dropout 0.0, and no bias adaptation. Adapters are applied to
the MLP projection modules \(\texttt{gate\_proj}\),
\(\texttt{up\_proj}\), and \(\texttt{down\_proj}\) under the causal language
modeling task configuration.

\paragraph{Evaluation.}
All evaluations use the held-out test split of the v6 pure-5 dataset, which
contains 214 examples: 53 from \(T_2\), 54 from \(T_3\), 80 from \(T_4\),
and 27 from \(T_5\). Accuracy is computed from normalized generated labels.
An output is counted as correct when it contains the target label, preferably
within the required \texttt{<answer>} tag; outputs without a valid label are
counted as invalid.

\paragraph{In-context learning prompt.}
For ICL, the model receives labeled examples as \((x,y)\) pairs followed by a
query integer. No parameters are updated.

\begin{tcolorbox}[
    title={ICL prompt},
    breakable,
    colback=gray!1,
    colframe=black!80,
    boxrule=0.5pt,
    arc=1mm,
    left=1mm,
    right=1mm,
    top=1mm,
    bottom=1mm
]
\footnotesize
\ttfamily
\raggedright

System:\par
You are given examples of an arbitrary
integer-to-label rule.
The labels are arbitrary class names, not the
mathematical sign of the integer.
\medskip

Infer the rule from the examples and answer
the query with exactly one tag:\par
\medskip

<answer>positive</answer>\par
or\par
<answer>negative</answer>.\par
\medskip

User:\par
Examples are shown as (x, y) pairs, where x is\par
an integer and y is its label.\par
\medskip

Use only the pattern in these pairs to label\par
the query integer.\par
\medskip

Examples:\par
(x\_1, y\_1)\par
...\par
(x\_m, y\_m)\par
\medskip

Assistant:\par
Understood.\par
\medskip

User:\par
Query:\par
x = \{x\_query\}\par
\medskip

Answer with exactly one tag:\par
<answer>positive</answer>\par
or\par
<answer>negative</answer>.
\end{tcolorbox}

The default \(W_1\) ICL context contains eight balanced examples: two from
\(T_2\), four from \(T_3\), and two from \(T_5\), yielding four
\texttt{positive} and four \texttt{negative} examples. We vary the number of
\(T_4\) context examples from 0 to 8 and evaluate both ordered and shuffled
context sequences. ICL uses deterministic decoding with batch size 4, a
maximum input length of 2048 tokens, and at most 16 generated tokens.

%% file: appendix/4_experimental_results.tex
\begin{figure*}[t!]
    \centering
    \includegraphics[width=1.0\linewidth]{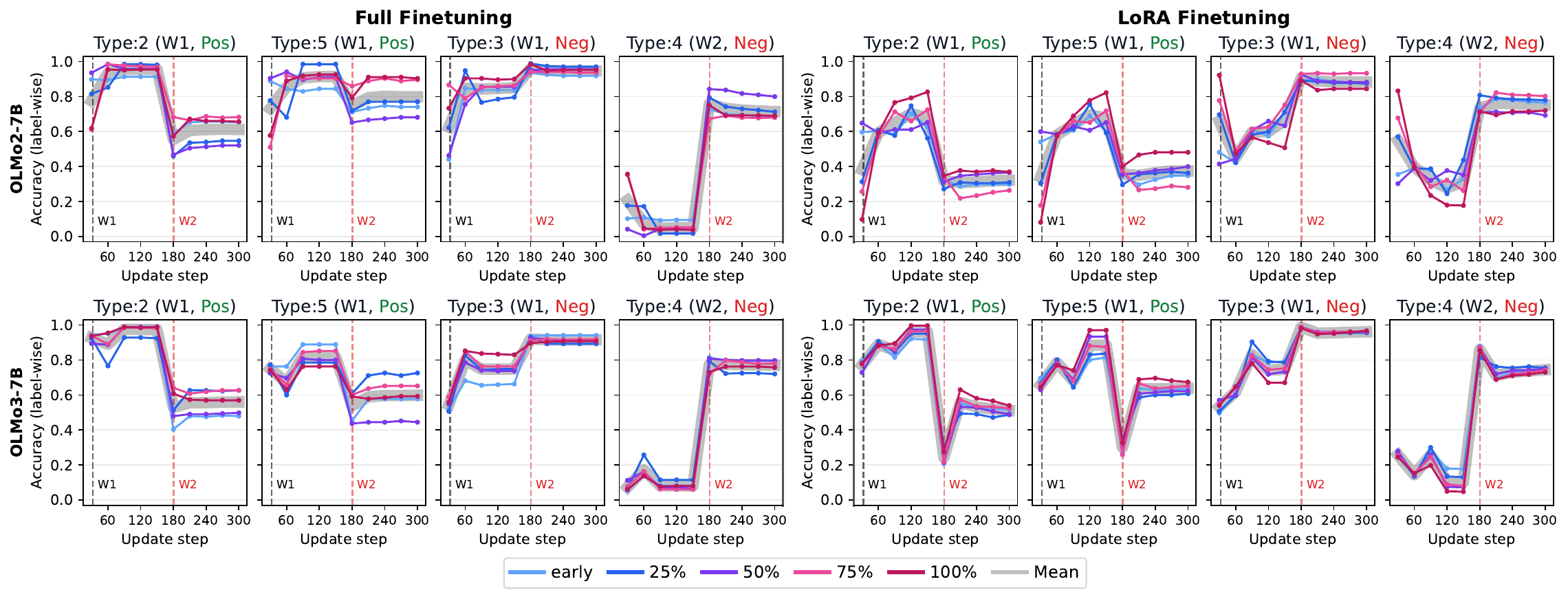}
    \caption{Type-wise accuracy trajectories across OLMo checkpoints under full and LoRA finetuning. The red dashed line marks the transition from \(W_1\) to \(W_2\), and the gray line shows the checkpoint mean.}
    \label{fig:appendix_checkpoint_tuning}
\end{figure*}

\section{Additional Experimental Results}
\label{appendix:experimental_results}
\paragraph{Overview.}
The main experiments establish the non-monotonic flip pattern at the
behavioral level. The appendix further examines properties predicted by the
Lifted State Hypothesis through four complementary analyses.

\begin{enumerate}
    \item \textbf{Model scale and checkpoints.}
    Does the generalization--revision pattern persist across model sizes?

   \item \textbf{Sequential revision versus joint learning.} Is it more effective to first learn the broad \(W_1\) rule and revise it under \(W_2\), or to learn all \(W_2\) cases jointly from the start? We compare sequential \(W_1 \rightarrow W_2\) learning with \(W_2\)-only learning using the same number of \(W_2\) updates.

    \item \textbf{Specificity under overlapping and non-overlapping controls.}
    Does the revision pattern involving \(G_5\) persist when \(G_5\) is
    constructed with or without numerical overlap with the other groups?

    \item \textbf{Activation and causal analysis.}
    Does \(W_1\) form shared MLP activation structure between \(G_2\) and
    \(G_4\), and do changes to these shared components explain the subsequent
    \(G_2\) decline through routing or output-value revision?
\end{enumerate}

\subsection{Model Scale and Checkpoints}

We examine whether the generalization--revision pattern persists across model
sizes and pretraining checkpoints. We compare Llama, Gemma-3, and Qwen3 at
multiple scales and evaluate five checkpoints of OLMo-2-7B and OLMo-3-7B
under full and LoRA finetuning. Figure~\ref{fig:scale_checkpoint_summary}
summarizes final accuracy averaged over all types, while
Figure~\ref{fig:appendix_checkpoint_tuning} presents the corresponding
accuracy across OLMo checkpoints.

As shown in Figure~\ref{fig:appendix_checkpoint_tuning}, the qualitative
pattern persists across checkpoints and update methods. During \(W_1\), the
models learn the shared positive behavior of \(\mathcal{G}_2\) and
\(\mathcal{G}_5\), while \(\mathcal{G}_4\) remains negative. After the
transition to \(W_2\), \(\mathcal{G}_4\) accuracy rises, whereas
\(\mathcal{G}_2\), and often \(\mathcal{G}_5\), declines. In contrast,
the unrelated \(\mathcal{G}_3\) is generally preserved.

\begin{figure}[t!]
    \centering
    \includegraphics[width=1.0\linewidth]{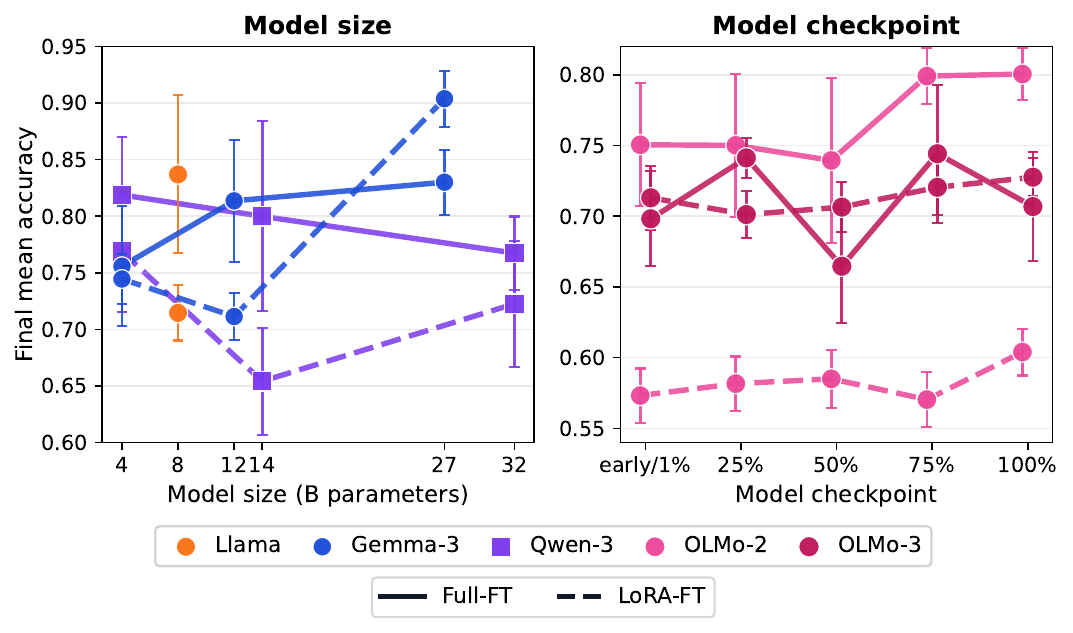}
    \caption{Final mean accuracy over all types across model scales and OLMo checkpoints. }
\label{fig:scale_checkpoint_summary}
    \label{fig:appendix_fig_size_and_checkpoints}
\end{figure}
Figure~\ref{fig:scale_checkpoint_summary} shows that final performance is not
monotonic in either model scale or checkpoint. Larger models improve in some
families and update settings, but no consistent scale advantage appears across
all configurations. Similarly, later OLMo checkpoints do not uniformly
outperform earlier ones, and the relative behavior of full and LoRA
finetuning varies across the two OLMo families. This variability suggests that
capacity alone does not determine how effectively a model separates revised
and preserved types. Instead, revision performance appears to depend on the
interaction between the pretrained representation and the update mechanism.

These results suggest that stronger models may form finer distinctions among
computation-relative types and remain more stable when subtype-specific
evidence is introduced. Greater scale or pretraining progress can therefore
reduce unnecessary coupling between the revised subtype and preserved parent
cases, although this advantage is not uniform across all models and update
methods.

\begin{table*}[t]
\centering

\textbf{Full Finetuning}\par
\vspace{2pt}

\resizebox{\textwidth}{!}{%
\begin{tabular}{lccccccccccccccccc}
\toprule
 & \multicolumn{1}{c}{Llama} & \multicolumn{3}{c}{Gemma} & \multicolumn{3}{c}{Qwen} & \multicolumn{5}{c}{OLMo 2} & \multicolumn{5}{c}{OLMo 3} \\
\cmidrule(lr){2-2} \cmidrule(lr){3-5} \cmidrule(lr){6-8} \cmidrule(lr){9-13} \cmidrule(lr){14-18}
Schedule & 8B & 4B & 12B & 27B & 4B & 14B & 32B & early & 25\% & 50\% & 75\% & 100\% & early & 25\% & 50\% & 75\% & 100\% \\
\midrule
W1+W2 & \textbf{0.837} & \textbf{0.756} & \textbf{0.814} & \textbf{0.830} & \textbf{0.819} & \textbf{0.800} & \textbf{0.767} & \textbf{0.751} & \textbf{0.750} & \textbf{0.739} & \textbf{0.799} & \textbf{0.801} & {0.698} & \textbf{0.741} & {0.665} & \textbf{0.744} & 0.707 \\
W2 & 0.808 & 0.697 & 0.792 & 0.797 & 0.737 & 0.690 & 0.729 & 0.595 & 0.605 & 0.610 & 0.600 & 0.594 & \textbf{0.718} & 0.707 & \textbf{0.735} & 0.733 & \textbf{0.733} \\
\midrule
\end{tabular}%
}
\textbf{LoRA Finetuning}\par
\vspace{1pt}
\resizebox{\textwidth}{!}{%
\begin{tabular}{lccccccccccccccccc}
\midrule
 & \multicolumn{1}{c}{Llama} & \multicolumn{3}{c}{Gemma} & \multicolumn{3}{c}{Qwen} & \multicolumn{5}{c}{OLMo 2} & \multicolumn{5}{c}{OLMo 3} \\
\cmidrule(lr){2-2} \cmidrule(lr){3-5} \cmidrule(lr){6-8} \cmidrule(lr){9-13} \cmidrule(lr){14-18}
Schedule & 8B & 4B & 12B & 27B & 4B & 14B & 32B & early & 25\% & 50\% & 75\% & 100\% & early & 25\% & 50\% & 75\% & 100\% \\
\midrule
W1+W2 & \textbf{0.715} & \textbf{0.745} & \textbf{0.711} & \textbf{0.904} & \textbf{0.769} & \textbf{0.654} & 0.723 & 0.573 & 0.582 & 0.585 & 0.570 & 0.604 & 0.713 & 0.701 & 0.706 & 0.720 & 0.728 \\
W2 & 0.650 & 0.678 & 0.668 & 0.858 & 0.718 & 0.624 & \textbf{0.771} & \textbf{0.612} & \textbf{0.602} & \textbf{0.620} & \textbf{0.630} & \textbf{0.623} & \textbf{0.721} & \textbf{0.717} & \textbf{0.731} & \textbf{0.727} & \textbf{0.730} \\
\bottomrule
\end{tabular}%
}
\caption{Final mean accuracy under sequential revision and joint learning,
averaged over types and seeds. Sequential revision is generally better for
Llama, Gemma, Qwen, and full-finetuned OLMo, whereas joint learning is better
for LoRA-finetuned OLMo.}
\label{tab:path_dependent_final_accuracy}
\end{table*}

\begin{figure*}[t!]
    \centering
    \includegraphics[width=1.0\linewidth]{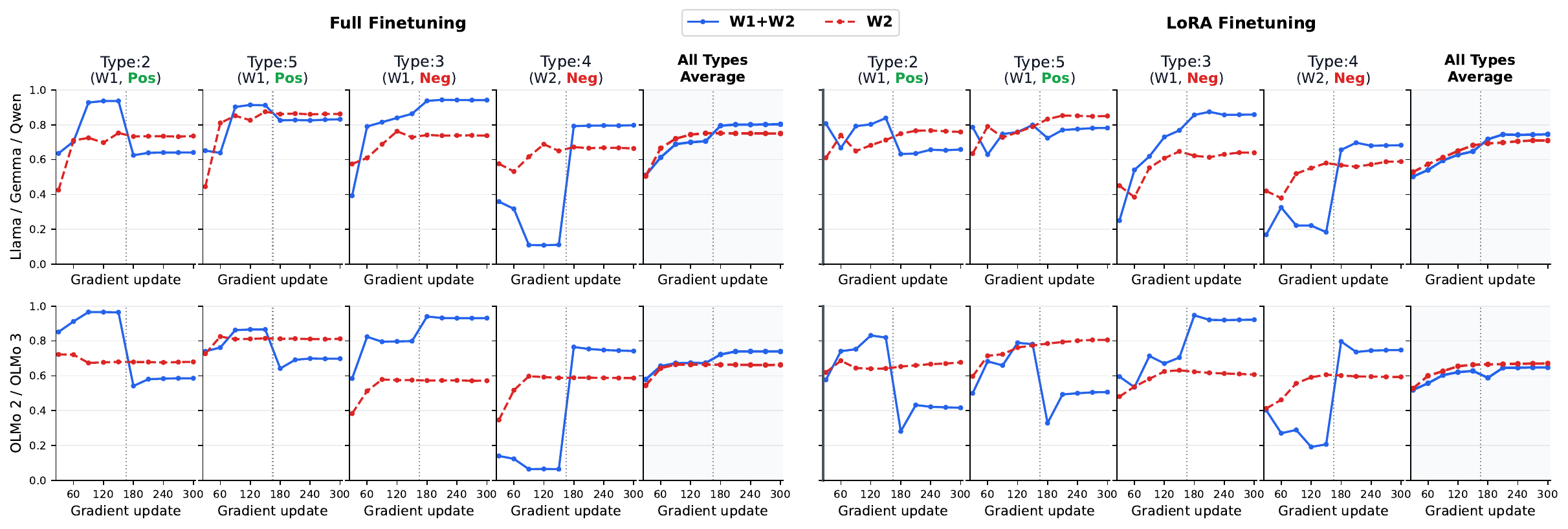}
\caption{Type-wise accuracy trajectories under sequential
\(W_1 \rightarrow W_2\) and \(W_2\)-only training for full and LoRA
finetuning. The dotted line marks the onset of \(W_2\) updates in the
sequential schedule, highlighting the subsequent flip and recovery dynamics.}
\label{fig:path_dependent_trajectories}
\end{figure*}

\subsection{Sequential Revision versus Joint Learning}

We compare two ways of learning the final \(W_2\) rule. In the joint-learning
condition, the model is trained directly on
\((T_2,T_3,T_4,T_5)\). In the sequential-revision condition, it first learns
\(W_1\) from \((T_2,T_3,T_5)\) and is then revised using \(W_2\). Both
conditions receive the same number of \(W_2\) updates. This comparison tests
whether first learning a broad rule helps or hinders the subsequent
acquisition of the subtype distinction required by \(W_2\).

Table~\ref{tab:path_dependent_final_accuracy} shows that, under full
finetuning, sequential revision achieves higher final mean accuracy in most
settings. This tendency appears across Llama, Gemma, Qwen, and the OLMo
checkpoints, although its magnitude varies by model. Under LoRA finetuning,
sequential revision is also generally stronger for Llama, Gemma, and most
Qwen models, with Qwen-32B as the main exception. The OLMo results depend
more strongly on the update method: sequential revision is usually better
under full finetuning, whereas joint learning performs better across all
checkpoints under LoRA.

Figure~\ref{fig:path_dependent_trajectories} clarifies how these final
differences emerge. Under joint learning, accuracy typically rises to a
moderate level and then stabilizes, suggesting early convergence with limited
subsequent reorganization. Sequential revision instead exhibits the
characteristic flip pattern: after the transition to \(W_2\), performance on
some previously generalized types declines, followed by partial or substantial
recovery as the model reorganizes its solution.

The flip should therefore not be interpreted solely as a learning failure.
In many settings, first acquiring the broad \(W_1\) rule and then revising it
produces better final performance than learning all \(W_2\) distinctions
jointly from the start. The initial phase may provide a computational scaffold
that captures the dominant regularity, while the later decline reflects the
cost of separating the revised subtype from cases that previously shared the
same computation. 

Sequential revision can thus be beneficial rather than purely interfering.
Learning a broad regularity first and refining it after subtype-specific
evidence arrives may be easier than discovering the broad rule and its
exception simultaneously. However, the OLMo LoRA results show that this
advantage is not universal and depends on both the model and the update
mechanism. Sequential revision may therefore provide a productive learning
curriculum, but its benefit depends on whether the model can effectively
reorganize the computation established during \(W_1\).

\begin{figure*}[t!]
    \centering
    \includegraphics[width=0.65\linewidth]{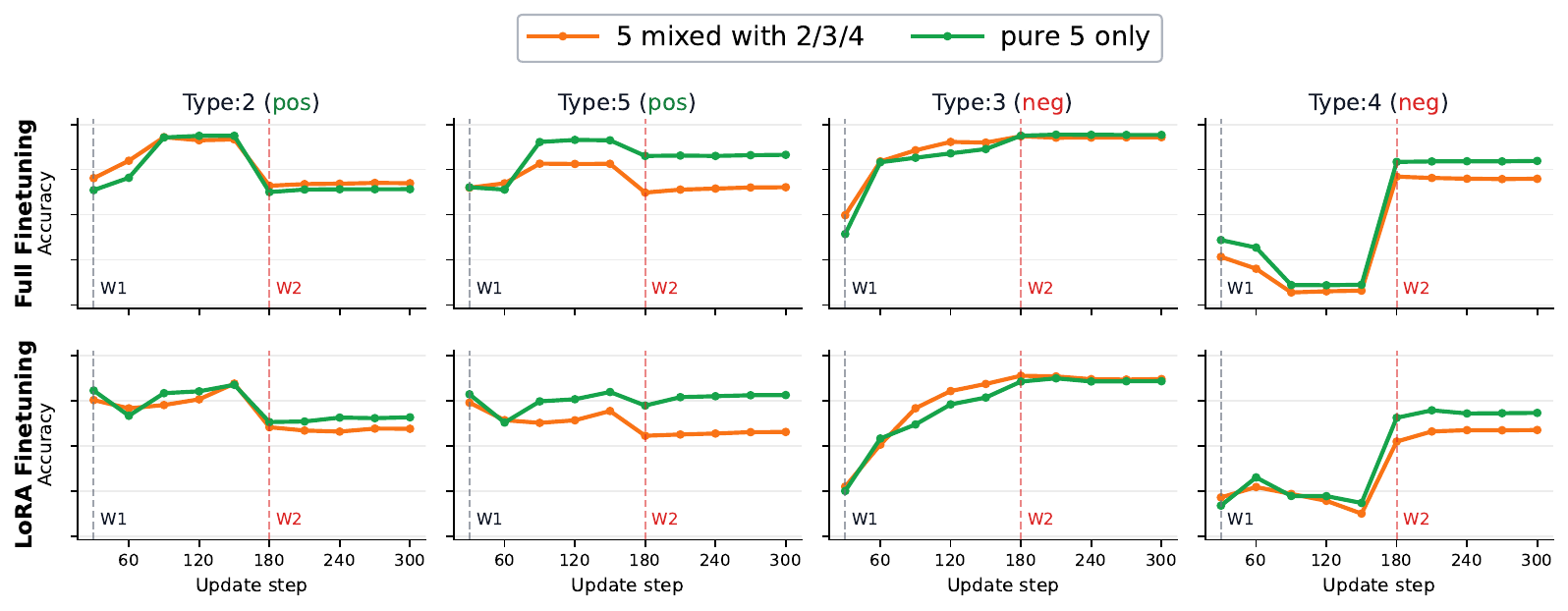}
    \caption{Accuracy trajectories averaged over Types 2, 3, 4, and 5 for mixed and pure \(\mathcal{G}_5\) under full and LoRA finetuning.}
    \label{fig:appendix_fig_data_mix}
\end{figure*}

\begin{figure*}[t!]
    \centering
    \includegraphics[width=1.0\linewidth, trim={0 0 0 12mm}, clip]{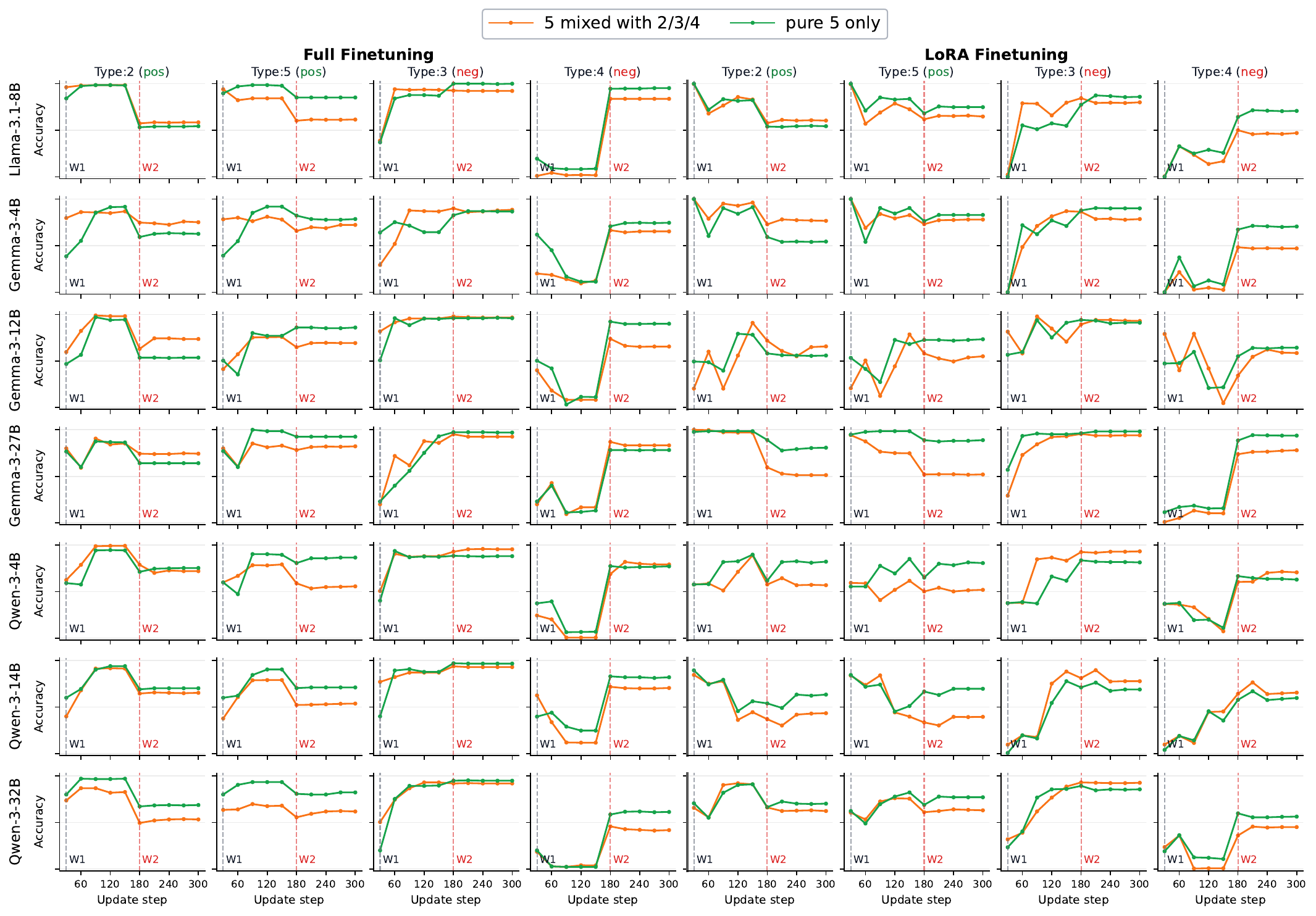}
    \caption{Model-wise accuracy trajectories comparing mixed and pure \(\mathcal{G}_5\) constructions across full and LoRA finetuning.}
    \label{fig:appendix_fig_data_mix_full}
\end{figure*}
\subsection{Effects of Subtype Overlap}

\paragraph{Instance overlap under parameter tuning.}
We examine whether the decline of \(\mathcal{G}_5\) during \(W_2\) revision
depends on its numerical overlap with the other types. We compare two
constructions:
\begin{itemize}
    \item \textbf{Mixed \(\mathcal{G}_5\):} \(\mathcal{G}_5^{\mathrm{mix}}=\{x\in\mathcal{X}\mid x\equiv 0\pmod{5}\}\).
    \item \textbf{Pure \(\mathcal{G}_5\):} \(\mathcal{G}_5^{\mathrm{pure}}=\mathcal{G}_5^{\mathrm{mix}}\setminus(\mathcal{G}_2\cup\mathcal{G}_3\cup\mathcal{G}_4)\).
\end{itemize}
This comparison tests whether shared instances make \(\mathcal{G}_5\) more
susceptible to collateral revision.

Figures~\ref{fig:appendix_fig_data_mix} and
\ref{fig:appendix_fig_data_mix_full} show that the mixed condition generally
suffers a larger decline in \(\mathcal{G}_5\) accuracy after the transition to
\(W_2\) than the pure condition. This pattern appears under both full and LoRA
finetuning. In most configurations, accuracy drops immediately after \(W_2\)
and then partially recovers, consistent with the model gradually separating
the preserved \(\mathcal{G}_5\) cases from the revised subtype. Gemma-3-12B
under full finetuning is a notable exception, as its \(\mathcal{G}_5\)
accuracy improves after \(W_2\). Because the pure condition also shows some
decline, numerical overlap strengthens the flip but is not its sole cause.

\begin{figure*}[ht!]
    \centering
    \includegraphics[width=0.8\linewidth]{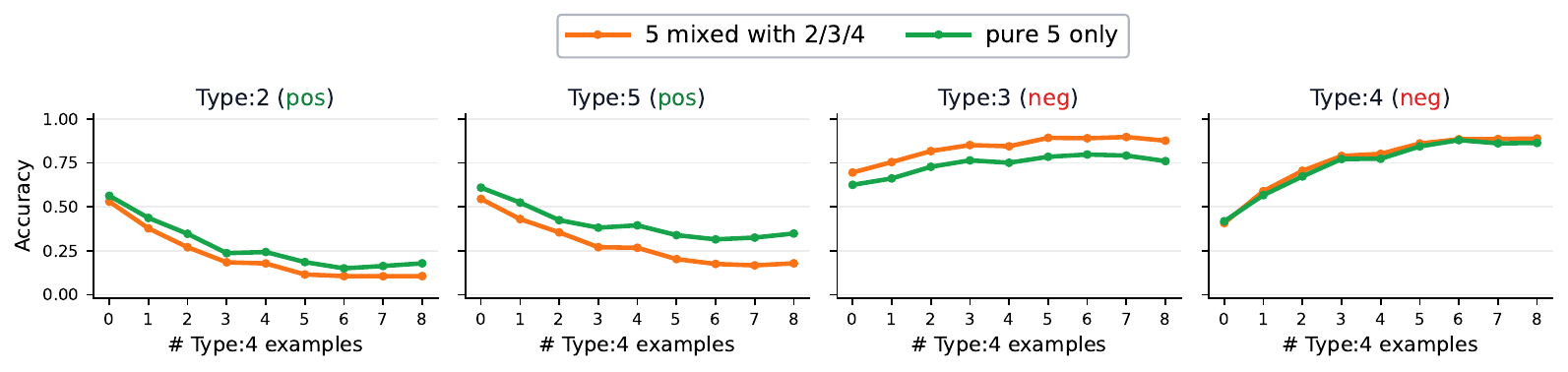}
    \caption{Mean accuracy over all models by Type 4 demonstrations for mixed and pure \(\mathcal{G}_5\) under ICL.}
    \label{fig:appendix_fig_data_mix_icl}
\end{figure*}

\begin{figure*}[ht!]
    \centering
    \includegraphics[width=1.0\linewidth, trim={0 0 0 12mm}, clip ]{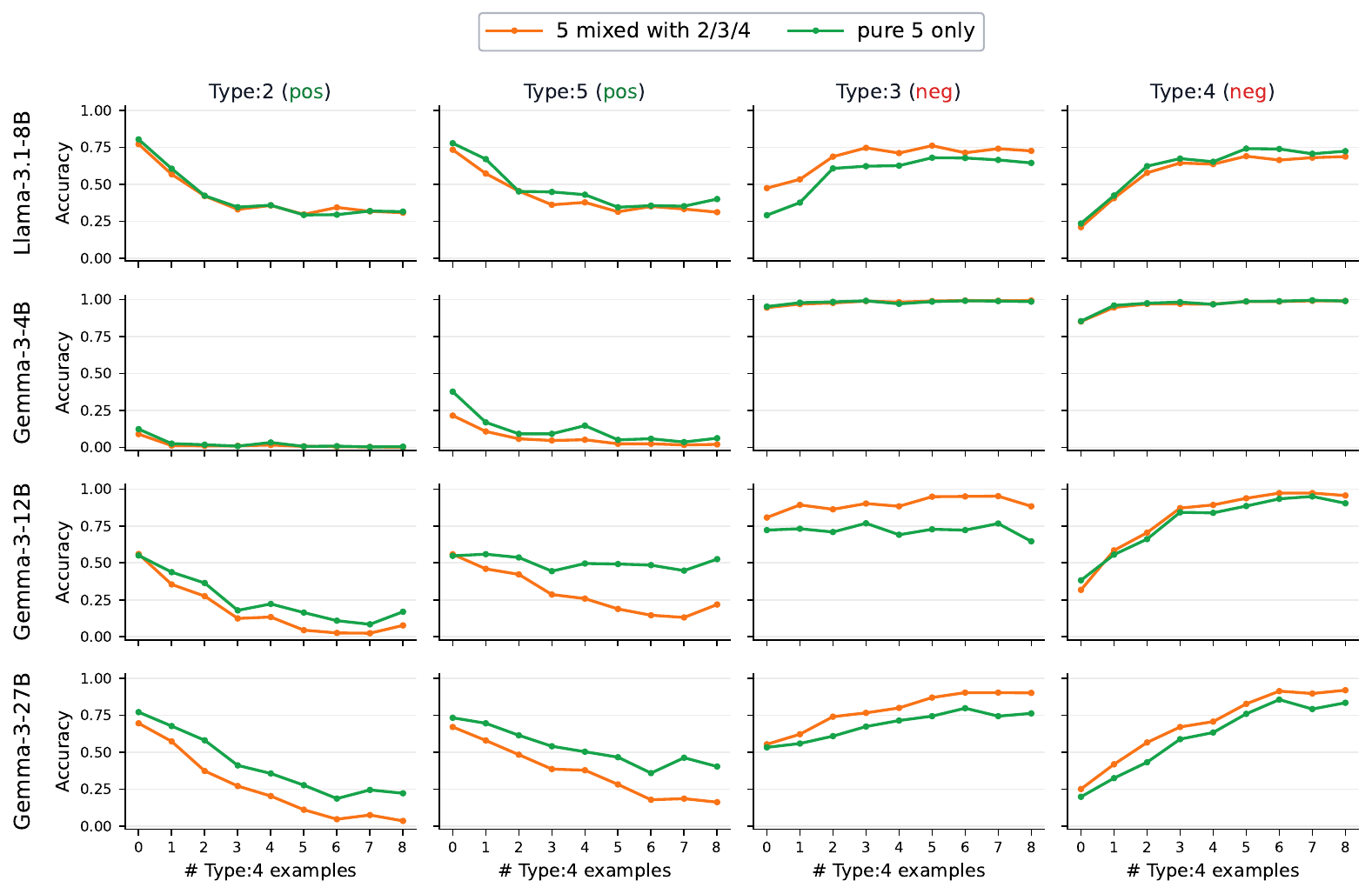}
    \caption{Model-wise accuracy by Type 4 demonstrations for mixed and pure \(\mathcal{G}_5\) under ICL.}
    \label{fig:appendix_fig_data_mix_icl_full}
\end{figure*}

\paragraph{Implications of instance overlap in in-context learning.}
Figures~\ref{fig:appendix_fig_data_mix_icl}
and~\ref{fig:appendix_fig_data_mix_icl_full} compare mixed Type~5 examples,
which also satisfy one of the Type~2--4 conditions, with pure Type~5 examples,
which satisfy none of them. Using pure Type~5 examples reduces the decline in
both Type~2 and Type~5 accuracy as Type~4 demonstrations are added. Literal
instance overlap therefore appears to contribute to the collateral effect of
contextual revision and to the coupling between these types.  However, pure Type~5 examples also produce lower Type~3 accuracy, despite
being disjoint from both Type~2 and Type~3. Removing overlap therefore does
not isolate a single grouping mechanism. It may also change which observable
properties, such as parity or other numerical regularities, become salient in
the context.

\paragraph{Limits of identifying the induced lifting.}
This pattern illustrates the difficulty of determining which cases the model
lifts into a shared computational type. The improved separation between
Type~2 and Type~5 is compatible with weaker overlap-based coupling, but the
simultaneous change in Type~3 suggests that other latent groupings may also be
involved. Parity is one possible explanation, because pure Type~5 and Type~3
examples are odd while Type~2 examples are even, but the results do not
uniquely establish it. More generally, several properties may induce similar changes in accuracy.  
Behavioral comparisons can therefore reveal shared computation without uniquely
identifying the rule that defines the group. Under the Lifted State
Hypothesis, detecting lifting may thus be easier than recovering the precise
latent type used by the model.


\begin{figure*}[t!]
    \centering
    \includegraphics[width=1.0\linewidth]{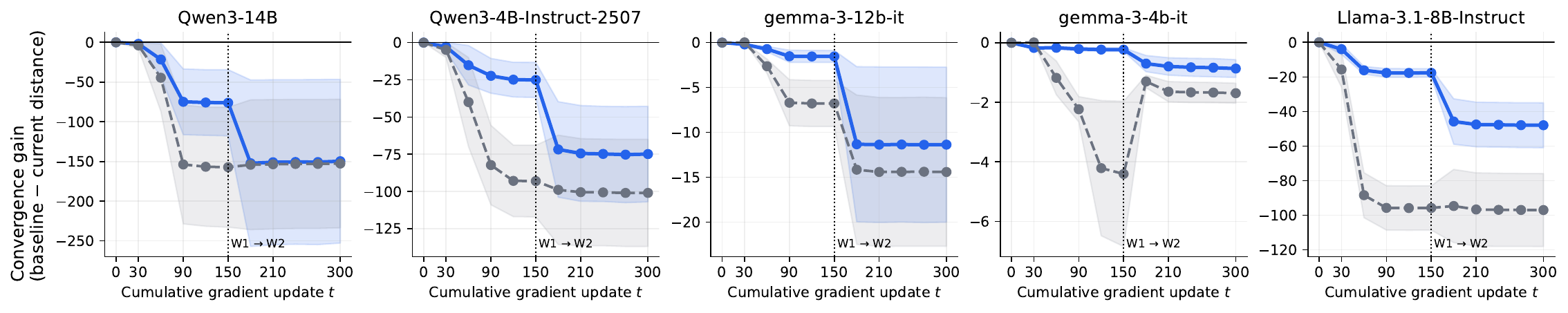}
    \caption{The blue curve shows the baseline-relative convergence gain of
    Type~2--Type~4, while the gray dashed curve shows the mean gain across
    the four cross-control pairs formed by pairing Type~2 or Type~4 with
    Type~3 or Type~5. During \(W_1\), both gains are negative, but the blue
    curve remains above the gray curve, indicating that Type~2--Type~4
    separates less than the controls despite the absence of absolute
    convergence. Lines and bands denote the five-seed mean and standard error.}
    \label{fig:appendix_fig_activation_1_change}
    \label{fig:appendix_fig_activation_1_change}
\end{figure*}

\begin{figure*}[t!]
    \centering
    \includegraphics[width=1.0\linewidth]{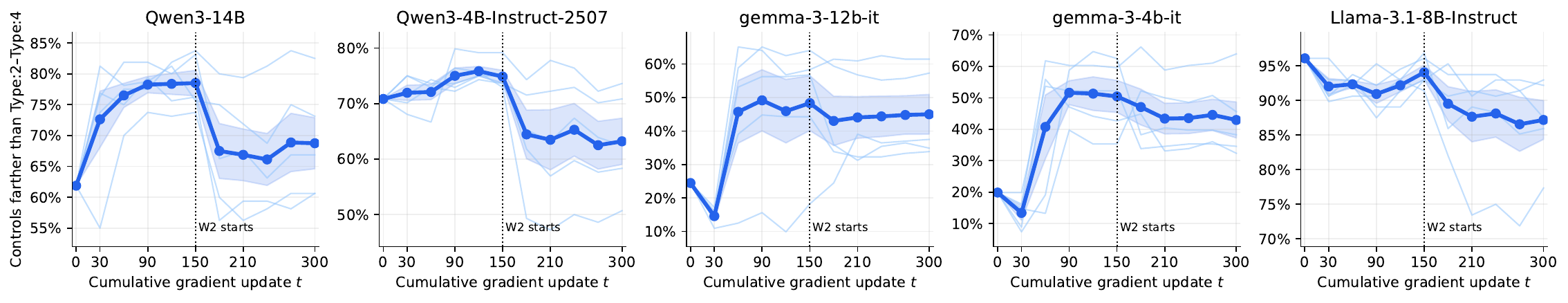}
    \caption{The y-axis reports the proportion of the four cross-control
    distances \((2,3)\), \((2,5)\), \((4,3)\), and \((4,5)\) that exceed
    the Type~2--Type~4 distance. Thin curves show individual seeds, while
    the thick curve and band denote the five-seed mean and standard error
    of the mean. Relative Type~2--Type~4 closeness generally increases
    during \(W_1\) and decreases after \(W_2\) begins.}
    \label{fig:appendix_fig_activation_2_relative}
\end{figure*}

\subsection{Activation analysis}
\label{appendix:activation_causal_analysis}

We examine how similarity among type-specific MLP activation patterns develops
during full fine-tuning and how this organization changes when one type must
be revised while a previously learned relation is retained. We refer to
Type~2 as the \emph{retained type}, because its \(W_1\) behavior should
remain unchanged in \(W_2\), and Type~4 as the \emph{revision type}, because
it receives conflicting supervision in \(W_2\).

\paragraph{Routing convergence across types.}
We first test whether the retained and revision types develop selectively
similar MLP activation patterns during \(W_1\), relative to the Type~3 and
Type~5 controls. Their exposure during \(W_1\) is asymmetric: the retained
type is directly trained, whereas the revision type is held out. Optimization
may therefore move the retained-type pattern across many neurons without a
matching change in the revision type. We consequently distinguish a decrease
in their absolute distance from the weaker but more relevant condition that
they remain closer to each other than to the control types. We then examine
whether this relation weakens when the revision type is trained toward a
different outcome in \(W_2\).

\paragraph{Observed activation dynamics.}
Figure~\ref{fig:appendix_fig_activation_1_change} shows that the absolute
distance between the retained and revision types generally increases during
\(W_1\). This is consistent with their asymmetric exposure: direct
optimization moves the retained-type activation pattern, while the held-out
revision type does not undergo the same supervised change. An increase in
their full-neuron distance therefore does not rule out a selective relation
between them.

The relevant comparison is whether this pair separates less than the four
cross-control pairs \((2,3)\), \((2,5)\), \((4,3)\), and \((4,5)\).
These control distances generally increase more than the distance between
the retained and revision types. Thus, the revision type remains relatively
closer to the retained type than the control types do, even though it does
not follow the retained type's full neuron-wise movement. The \(W_1\) update
therefore produces selective relative closeness rather than uniform
activation similarity.

Figure~\ref{fig:appendix_fig_activation_2_relative} makes this comparison
explicit. Its y-axis, \textbf{Controls farther than Type~2--Type~4},
reports the percentage of the four cross-control pairs whose activation
distance exceeds the retained--revision distance. A value of \(0\%\) means
that the retained and revision types are not closer than any control pair,
whereas \(100\%\) means that they are closer than all four. This percentage
increases during \(W_1\) in both Qwen models and both Gemma models.
Llama-3.1-8B-Instruct begins at a high relative-closeness level and largely
maintains it. The main \(W_1\) effect is therefore relative convergence:
the retained and revision types may separate in absolute activation space
while remaining increasingly close relative to the controls.

This relative advantage decreases immediately after \(W_2\) begins in all
five models. Because \(W_2\) trains the revision type toward a conflicting
outcome while preserving the retained type, the decrease is consistent with
the model differentiating their previously related activation patterns. The
extent of this separation varies across models, and some retain part of the
relative closeness established during \(W_1\).
\begin{figure*}[t!]
    \centering
    \includegraphics[width=1.0\linewidth]{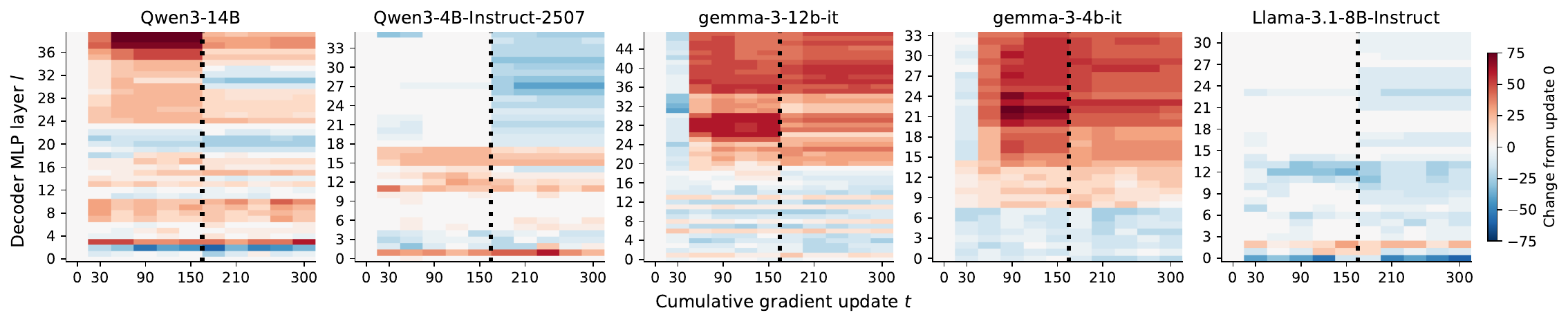}
    \caption{Layerwise change in Type~2--Type~4 relative closeness from
    update \(0\). Red indicates that the pair becomes closer than more
    cross-control pairs, whereas blue indicates the opposite. Relative
    convergence during \(W_1\) is broad across the middle and upper Gemma
    layers, more localized in Qwen, and limited in Llama. The transition to
    \(W_2\) weakens or reverses this relation in model-specific layers rather
    than uniformly across the network.}
    \label{fig:appendix_fig_activation_3_layer}
\end{figure*}

\paragraph{Layerwise relative convergence.}
Model-level trajectories average over all MLP layers and can obscure where
the Type~2--Type~4 relation emerges or disappears. We therefore examine
whether the relative convergence observed during \(W_1\) is distributed
across the network or concentrated in specific layers. Here, relative
convergence means that Type~4 remains closer to the trained Type~2 than
Types~3 and~5 do, even when their absolute distance increases.

Figure~\ref{fig:appendix_fig_activation_3_layer} shows strong layer
dependence. In Qwen3-14B, the \(W_1\) increase is concentrated in the middle
and upper layers, with \(W_2\) weakening or reversing it in only some layers.
Qwen3-4B-Instruct-2507 shows gains mainly in lower and middle layers, while
several upper layers become increasingly negative after \(W_2\).
Llama-3.1-8B-Instruct exhibits smaller and less consistent changes, partly
because its initial relative closeness is already high.

Both Gemma models show broader convergence across middle and upper layers
during \(W_1\). Many positive changes persist into \(W_2\), although some
are weakened or locally reversed. Overall, the Type~2--Type~4 relation
emerges in model-specific layers and is selectively reorganized when
Type~4 receives conflicting supervision.

\paragraph{Technical details.}
At update \(0\) and each observation point from updates \(30\) to \(300\),
we evaluate the same probe examples, using equal numbers from the retained
type (Type~2), the revision type (Type~4), and the two control types
(Types~3 and~5). For decoder layer \(l\), the routing pattern of example
\(x\) is represented by its signed post-gated activation at the final prompt
token, immediately before \texttt{down\_proj}:
\begin{equation}
    a_l(x,t)
    =
    \phi\!\left(W_{\mathrm{gate},l}(t)h_l(x)\right)
    \odot
    W_{\mathrm{up},l}(t)h_l(x)
    \in\mathbb{R}^{d_{\mathrm{ff}}}.
\end{equation}
This representation retains the magnitude and sign of each neuron-wise MLP
signal. Rather than using thresholded activation overlap, we compare the
patterns with a cross-validated diagonal-Mahalanobis distance. For each type
\(g\in\{2,3,4,5\}\), its probe examples are split once into two equally
sized, disjoint subsets \(A\) and \(B\). Let
\(\mu^A_{g,l}(t)\) and \(\mu^B_{g,l}(t)\) denote their mean activation
vectors. For two types \(g\) and \(h\), we define
\begin{equation}
    D_{gh,l}(t)
    =
    \frac{1}{d_{\mathrm{ff}}}
    \bigl(\mu^A_{g,l}(t)-\mu^A_{h,l}(t)\bigr)^\top
    W_l
    \bigl(\mu^B_{g,l}(t)-\mu^B_{h,l}(t)\bigr).
\end{equation}
Here, \(W_l\) is a diagonal precision matrix estimated from neuron-wise
residual variances on a separate baseline set, with \(10\%\) shrinkage
toward the layerwise median variance. The same \(W_l\) is used at every
update. Smaller \(D_{gh,l}(t)\) indicates more similar routing patterns.
Because the two type differences are estimated from independent subsets,
sampling noise can produce small negative values when the true distance is
near zero. We retain them to avoid the positive bias introduced by
truncation. These distances underlie
Figures~\ref{fig:appendix_fig_activation_1_change}--\ref{fig:appendix_fig_activation_3_layer}.

The baseline-relative convergence of the retained--revision pair,
corresponding to Type~2--Type~4, is
\begin{equation}
    G_{24,l}(t)
    =
    D_{24,l}(0)-D_{24,l}(t).
\end{equation}
We compare it with the four cross-control pairs formed by pairing either
the retained or revision type with a control type:
\(
\mathcal{C}=\{(2,3),(2,5),(4,3),(4,5)\}
\).
Their relative selectivity is
\begin{equation}
    S_{24,l}(t)
    =
    G_{24,l}(t)
    -
    \frac{1}{4}
    \sum_{(g,h)\in\mathcal{C}}
    \left[D_{gh,l}(0)-D_{gh,l}(t)\right].
\end{equation}
We call the change \emph{selective convergence} only when
\(G_{24,l}(t)>0\) and \(S_{24,l}(t)>0\): the retained and revision types
must become closer in absolute terms and more strongly than the controls.
Figure~\ref{fig:appendix_fig_activation_1_change} first averages these
values across decoder layers within each seed and then reports the mean and
standard error of the mean over five seeds.

Because the retained and revision types are exposed asymmetrically during
\(W_1\), we also measure their closeness relative to the control pairs:
\begin{equation}
    R_l(t)
    =
    \frac{1}{4}
    \sum_{(g,h)\in\mathcal{C}}
    \mathbb{I}\!\left[D_{24,l}(t)<D_{gh,l}(t)\right].
\end{equation}
Here, \(R_l(t)=0\%\) means that the retained--revision pair is not closer
than any control pair, whereas \(R_l(t)=100\%\) means that it is closer
than all four. Figure~\ref{fig:appendix_fig_activation_2_relative} averages
\(R_l(t)\) across layers within each seed and reports the seed trajectories
and their mean. Figure~\ref{fig:appendix_fig_activation_3_layer} instead
reports
\begin{equation}
    \Delta R_l(t)
    =
    100\left[R_l(t)-R_l(0)\right]
\end{equation}
for each layer and update, averaged over seeds. Positive values mean that
the retained--revision pair is closer than more control pairs than at
update \(0\), whereas negative values mean the opposite. Both measures are
descriptive rather than statistical pass criteria.